\documentclass[bsl]{asl}
\usepackage{xurl}
\usepackage{enumitem}
\usepackage{rotate}
\usepackage{float}
\usepackage[hang]{subfigure}
\usepackage{comment}
\usepackage{bbm}
\usepackage{tikz}
\usepackage{caption}
\theoremstyle{plain}
\newtheorem{thm}{Theorem}[section] 
\newtheorem{lem}[thm]{Lemma} 
\newtheorem{prop}[thm]{Proposition}

\theoremstyle{definition} 
\newtheorem{defn}{Definition}[section] 
 
\newtheorem{exmp}{Example}[section]

\theoremstyle{remark} 
\newtheorem*{rem}{Remark}

\def\confname{}
\def\copyrightyear{2000}
\newcommand{\Luka}{\L{}ukasiewicz }
\newcommand{\Ccal}{\mathcal{C}}
\newcommand{\Acal}{\mathcal{A}}
\newcommand{\Ical}{\mathcal{I}}
\newcommand{\Ncal}{\mathcal{N}}

\newcommand{\dt}{\text{d}}
\newcommand\figref{Figure~\ref}
\newcommand\secref{Section~\ref}
\newcommand\tblref{Table~\ref}
\newcommand\defref{Definition~\ref}
\newcommand\exmpref{Example~\ref}
\newcommand\thmref{Theorem~\ref}
\newcommand\lemref{Lemma~\ref}
\newcommand\propref{Proposition~\ref}
\usepackage[T1]{fontenc}
\title{Cellular Automata, Many-Valued Logic, and Deep Neural Networks}

\author{Yani Zhang}
\revauthor{Zhang, Yani}
\address{Chair for Mathematical Information Science\\ ETH Zurich\\ Switzerland}
\email{yanizhang@mins.ee.ethz.ch}
\author{Helmut Bölcskei}
\revauthor{Bölcskei, Helmut}
\thanks{H. Bölcskei gratefully acknowledges support by the Lagrange Mathematics and Computing Research Center, Paris, France.}
\address{Chair for Mathematical Information Science\\ ETH Zurich\\Switzerland} 
\email{hboelcskei@ethz.ch}

\begin{document}

\begin{abstract}
We develop a theory characterizing the fundamental capability of deep neural networks to learn, from evolution traces, the logical rules governing the behavior
of cellular automata (CA). This is accomplished by first establishing a novel connection between CA and {\L}ukasiewicz propositional logic. While binary CA have been known for decades to essentially perform operations in Boolean logic, no such relationship exists for general CA. We demonstrate that many-valued (MV) logic, specifically {\L}ukasiewicz propositional logic, constitutes a suitable language for characterizing general CA as logical machines. 
This is done by interpolating CA transition functions to continuous piecewise linear functions, which, by virtue of the McNaughton theorem, yield formulae in MV logic characterizing the CA. Recognizing that deep rectified linear unit (ReLU) networks realize continuous piecewise linear functions, it follows that these formulae are naturally extracted from CA evolution traces by deep ReLU networks. A corresponding algorithm together with a software implementation is provided.
%ncoded in deep ReLU networks.
%These piecewise linear functions are then demonstrated 
%Then, it is shown that every CA transition rule can be realized by deep ReLU networks. This is achieved by first deriving a universal representation theorem %of CA transition rules in terms of many-valued logical formulae and then showing that ReLU networks naturally realize statements in many-valued algebras. 
Finally, we show that the dynamical behavior of CA can be
%Moreover, we show that CA, which are discrete-time dynamical systems, can be 
realized by recurrent neural networks. 
%Finally, for neural networks trained on CA evolution data, we present an algorithm for extracting the underlying transition rule in the explicit form of many-%valued logical formulae.  
\end{abstract}
\maketitle

%%%%%%%%%%%%%%%%%%%%%%%%%%%%% section 1 %%%%%%%%%%%%%%%%%%%%%%%%%%%%%%%%%%%%%%%%%
\section{Introduction}\label{sec:Intro}
Neural networks were originally introduced as a computational model whose functionality imitates that of the human brain \cite{mcculloch1943logical}. Machine learning based on neural networks has achieved state-of-the-art results in numerous applications, such as pattern recognition \cite{liang2015recurrent}, game intelligence \cite{silver2016mastering}, or protein structure prediction \cite{jumper2021highly}. Informally speaking, a neural network consists of layers of nodes that are connected by weighted edges. In practice, the network topology and weights are learned through training on data either in a supervised \cite{Kotsiantis2007supervised} or an unsupervised \cite{radford2015unsupervised} manner.
%and store the ``knowledge'' extracted from the training data.

%In this way, the network learns how to perform certain tasks based on given training data (e.g., images that have been classified, unlabeled images with %distinguishing features, or text sequences) and it stores the learned knowledge in its topology and weights. 

An ability ingrained in the human brain is to draw inferences from data it is presented with and apply what has been learned to new problems. In machine learning parlance, this aspect is usually referred to as generalization or extrapolation. For example, given the geometric sequence of integers $1,2,4,8,\ldots$, one would conclude that the rule generating the sequence is: The next number is obtained by multiplying the present number by $2$. Based on this extracted logical rule, the entire sequence can now be calculated. Such logical reasoning tasks, which appear natural to humans, may be challenging to learn for neural networks. Moreover, even if the network had learned, e.g., the mechanism generating the geometric sequence above, it would have the corresponding logical rule encoded in its topology and weights and it is unclear how to extract the rule and present it in a manner accessible to humans.

In this paper, we report an attempt at building a theory characterizing the fundamental capability of neural networks to learn logical rules from data. This immediately leads to the question, ``What is the logical structure behind data?'' Here, we shall make a first step towards developing the corresponding research program. Concretely, we consider data generated by cellular automata\footnote{We will abbreviate both the singular ``cellular automaton'' and the plural form ``cellular automata'' as ``CA''.} (CA). Abstractly speaking, a CA is a discrete dynamical system evolving on a regular lattice whose points take values in a finite set. Starting from an initial configuration, all lattice points change their states at synchronous discrete time steps according to a transition function that takes the present values of the lattice point under consideration and its neighbors as inputs. The evolution of the lattice point states over time furnishes the data sequence we shall be interested in. Now, assume that the underlying CA is binary, i.e., the state set is given by $\{0,1\}$. In this case the transition function governing the CA evolution is a Boolean function \cite[pp.~29-31]{voorhees1996computational}, which can be viewed 
%can then be expressed through a Boolean function taking lattice cell values as inputs. 
%This Boolean function can 
%be interpreted 
as determining the logical structure behind the data sequence generated by the CA. 

The question we shall ask is whether one can extract the logical rule governing a general (i.e., not only binary) CA by training a neural network on data generated by the CA. While the connection between binary CA and Boolean logic has been known for decades \cite{wolfram1983statistical}, \cite{voorhees1996computational}, to the best of our knowledge no such relationship has been reported in the literature for CA with arbitrary state sets. One of the main conceptual contributions of this paper is to demonstrate that many-valued\footnote{With slight abuse of terminology, we shall use the term MV logic to refer to \Luka propositional logic.} (MV) logic is a suitable language for interpreting general CA as logical machines. We then show that 
%Specifically, we shall consider CA with general (i.e., not only binary) state sets and prove 
all possible transition functions can be realized by deep ReLU networks\footnote{ReLU stands for the Rectified Linear Unit nonlinearity, defined as $x\mapsto \max\{0,x\}$.}, which, in turn, are found to naturally express statements in MV logic.
%This will be accomplished by first expressing CA transition rules in terms of many-valued\footnote{With slight abuse of notation, in this paper we use MV %logic to refer to \Luka propositional logic.} (MV) logical formulae \cite{cignoli2013algebraic} and then showing that ReLU networks naturally realize %statements in MV logic. 
The dynamical system component of CA is demonstrated to be realizable by recurrent neural networks (RNNs). Finally, we propose and analyze a procedure for extracting the logical formulae behind CA transition functions from neural networks trained on corresponding CA evolution data. 

\textit{Notation:} $\mathbbm{1}_{ \{\cdot \}}$ denotes the truth function which takes the value $1$ if the statement inside $\{\cdot\}$ is true and equals $0$ otherwise. $0_N$ stands for the $N$-dimensional column vector with all entries equal to $0$. $\mathbb{N}_0:= \mathbb{N} \cup \{0\}$ and $|A|$ stands for the cardinality of the set $A$. $||\cdot||_1$ is the $\ell_1$-norm.

\subsection{Cellular automata} 
CA were invented in the 1940s by von Neumann \cite{von2017general} and Ulam \cite{ulam1952random} in an effort to build models that are capable of universal computation and self-reproduction. Von Neumann's conceptualization emphasized the aspect of self-reproduction, while Ulam suggested the use of finite state machines on two-dimensional lattices. A widely known CA is the two-dimensional Game of Life devised by Conway \cite{gardner1970fantastic}. 
Despite the simplicity of its rules, the Game of Life exhibits remarkable behavioral complexity and has therefore attracted widespread and long-standing interest. We begin by briefly reviewing the Game of Life.

Consider an infinite two-dimensional grid of square cells centered on the points of an underlying lattice (symbolized by dashed lines in~\figref{fig:neighborhood}). Initially, each cell (or equivalently lattice point) is in one of two possible states, namely ``live'' or ``dead''. Each cell has eight neighbors, taking into account horizontal, vertical, and diagonal directions (see~\figref{fig:neighborhood}). For a given cell, we refer to the set made up of the cell itself and its neighbors as the neighborhood of the cell. The cells change their states synchronously at discrete time steps, all following the same rule given by: 
\begin{enumerate}
    \item Every live cell with two or three live neighbors stays live; other live cells turn into dead cells.
    \item Every dead cell with three live neighbors becomes live; other dead cells stay dead. 
\end{enumerate} 
\begin{figure}[H]
\centering
\resizebox{2.5cm}{!}{

\begin{tikzpicture}[x=1cm, y=1cm, >=stealth, scale=1, transform shape]

  \node at (0.75, 0.75) {$z$};
 
      \draw[dashed] (-0.25,0.25) -- (1.75,0.25);
   %   \draw[thick] (-0.25,0.5) -- (1.75,0.5);
          \draw[dashed] (-0.25,0.75) -- (1.75,0.75);
  %   \draw[thick](-0.25,1) -- (1.75,1);
         \draw[dashed] (-0.25,1.25) -- (1.75,1.25);
  %   \draw[thick] (-0.25,1.5) -- (1.75,1.5);
%  \draw[thick] (0, -0.25) -- (0, 1.75);   
 %  \draw[thick] (0.5, -0.25) -- (0.5, 1.75);   
%\draw[thick] (1, -0.25) -- (1, 1.75); 
\draw[dashed] (0.25, -0.25) -- (0.25, 1.75); 
\draw[dashed] (0.75, -0.25) -- (0.75, 1.75); 
\draw[dashed] (1.25, -0.25) -- (1.25, 1.75); 
% \draw[thick] (1.5, -0.25) -- (1.5, 1.75);   
   
 \draw[line width=0.5mm] (0,0) -- (1.5,0);
  \draw[line width=0.5mm] (0,0.5) -- (1.5,0.5);
    \draw[line width=0.5mm] (1.5,0) -- (1.5,1.5);
    \draw[line width=0.5mm] (0,1) -- (1.5,1);
 \draw[line width=0.5mm] (1.5,1.5) -- (0,1.5);  
  \draw[line width=0.5mm] (0,1.5) -- (0,0); 
    \draw[line width=0.5mm] (0.5,1.5) -- (0.5,0); 
       \draw[line width=0.5mm] (1,1.5) -- (1,0); 
\end{tikzpicture}
}
\caption{Cell $z$ and its neighborhood.}\label{fig:neighborhood}
\end{figure}
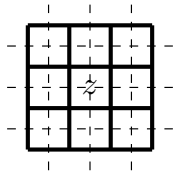
This process repeats at discrete time steps, forming an evolution of the system. We refer to \figref{fig:evolution} for an example illustration. As already mentioned, despite the simplicity of the transition rule, the evolutions induced by different initial configurations exhibit rich and complex long-term behavior. It is shown in \cite{berlekamp1982winning}, \cite{wainwright1974life}, \cite{rendell2002turing}, for example, that given an initial configuration with a finite number of live cells, it is an undecidable problem whether all these live cells would eventually die out. We refer to~\cite{adamatzky2010game} for an in-depth discussion of the behavioral patterns of the Game of Life and proceed to formally define CA. 
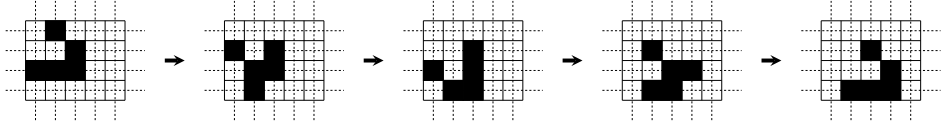
\begin{figure}[H]
\centering
\resizebox{12.5cm}{!}{

\begin{tikzpicture}[x=1cm, y=1cm, >=stealth, scale=1, transform shape]

   \draw[] (0,1) -- (0,5);   \draw[dashed] (0.5,0) -- (0.5,6);
  \draw[] (1,1) -- (1,5);\draw[dashed] (1.5,0) -- (1.5,6);
    \draw[] (2,1) -- (2,5); \draw[dashed] (2.5,0) -- (2.5,6);
   \draw[] (3,1) -- (3,5); \draw[dashed] (3.5,0) -- (3.5,6);
    \draw[] (4,1) -- (4,5); \draw[dashed] (4.5,0) -- (4.5,6);
     \draw[] (5,1) -- (5,5);
  \draw[] (0,1) -- (5,1);   \draw[dashed] (-1,1.5) -- (6,1.5);   
    \draw[] (0,2) -- (5,2);  \draw[dashed] (-1,2.5) -- (6,2.5);   
       \draw[] (0,3) -- (5,3);  \draw[dashed] (-1,3.5) -- (6,3.5);   
         \draw[] (0,4) -- (5,4); \draw[dashed] (-1,4.5) -- (6,4.5);   
        \draw[] (0,5) -- (5,5); 
        \draw[fill=black]  (1,4) -- (1,5) -- (2,5) -- (2,4) -- cycle;
      \draw[fill=black]  (2,3) -- (2,4) -- (3,4) -- (3,3) -- cycle;  
     \draw[fill=black]  (0,2) -- (0,3) -- (3,3) -- (3,2) -- cycle;      
  \draw[-stealth, line width=2mm] (7,3) -- (8,3) node[anchor=west] {};

  \draw[] (10,1) -- (10,5);\draw[dashed] (10.5,0) -- (10.5,6); 
\draw[] (11,1) -- (11,5); \draw[dashed] (11.5,0) -- (11.5,6); 
 \draw[] (12,1) -- (12,5); \draw[dashed] (12.5,0) -- (12.5,6); 
 \draw[] (13,1) -- (13,5); \draw[dashed] (13.5,0) -- (13.5,6); 
\draw[] (14,1) -- (14,5); \draw[dashed] (14.5,0) -- (14.5,6); 
 \draw[] (15,1) -- (15,5);    
 \draw[] (10,1) -- (15,1);   \draw[dashed] (9,1.5) -- (16,1.5);   
\draw[] (10,2) -- (15,2);  \draw[dashed] (9,2.5) -- (16,2.5);   
 \draw[] (10,3) -- (15,3); \draw[dashed] (9,3.5) -- (16,3.5);   
 \draw[] (10,4) -- (15,4);\draw[dashed] (9,4.5) -- (16,4.5);   
 \draw[] (10,5) -- (15,5); 
  \draw[fill=black]  (12,3) -- (12,4) -- (13,4) -- (13,3) -- cycle;  
  \draw[fill=black]  (10,3) -- (10,4) -- (11,4) -- (11,3) -- cycle; 
\draw[fill=black]  (11,2) -- (11,3) -- (13,3) -- (13,2) -- cycle;  
\draw[fill=black]  (11,1) -- (11,2) -- (12,2) -- (12,1) -- cycle;  
\draw[-stealth, line width=2mm] (17,3) -- (18,3) node[anchor=west] {};

 \draw[] (20,1) -- (20,5); \draw[dashed] (20.5,0) -- (20.5,6); 
\draw[] (21,1) -- (21,5); \draw[dashed] (21.5,0) -- (21.5,6); 
 \draw[] (22,1) -- (22,5); \draw[dashed] (22.5,0) -- (22.5,6); 
 \draw[] (23,1) -- (23,5);\draw[dashed] (23.5,0) -- (23.5,6); 
\draw[] (24,1) -- (24,5); \draw[dashed] (24.5,0) -- (24.5,6); 
 \draw[] (25,1) -- (25,5);  
 \draw[] (20,1) -- (25,1);  \draw[dashed] (19,1.5) -- (26,1.5);   
\draw[] (20,2) -- (25,2);   \draw[dashed] (19,2.5) -- (26,2.5);   
 \draw[] (20,3) -- (25,3);   \draw[dashed] (19,3.5) -- (26,3.5);   
 \draw[] (20,4) -- (25,4); \draw[dashed] (19,4.5) -- (26,4.5);   
 \draw[] (20,5) -- (25,5); 
 \draw[fill=black]  (20,2) -- (20,3) -- (21,3) -- (21,2) -- cycle; 
\draw[fill=black]  (21,1) -- (21,2) -- (22,2) -- (22,1) -- cycle;  
\draw[fill=black]  (22,1) -- (22,4) -- (23,4) -- (23,1) -- cycle;
\draw[-stealth, line width=2mm] (27,3) -- (28,3) node[anchor=west] {};

\draw[] (30,1) -- (30,5);\draw[dashed] (30.5,0) -- (30.5,6); 
\draw[] (31,1) -- (31,5); \draw[dashed] (31.5,0) -- (31.5,6); 
 \draw[] (32,1) -- (32,5); \draw[dashed] (32.5,0) -- (32.5,6); 
 \draw[] (33,1) -- (33,5); \draw[dashed] (33.5,0) -- (33.5,6); 
\draw[] (34,1) -- (34,5); \draw[dashed] (34.5,0) -- (34.5,6); 
 \draw[] (35,1) -- (35,5);      
 \draw[] (30,1) -- (35,1);  \draw[dashed] (29,1.5) -- (36,1.5);   
\draw[] (30,2) -- (35,2);  \draw[dashed] (29,2.5) -- (36,2.5);   
 \draw[] (30,3) -- (35,3); \draw[dashed] (29,3.5) -- (36,3.5);   
 \draw[] (30,4) -- (35,4);\draw[dashed] (29,4.5) -- (36,4.5);   
 \draw[] (30,5) -- (35,5);  
 \draw[fill=black]  (31,3) -- (31,4) -- (32,4) -- (32,3) -- cycle;
 \draw[fill=black]  (31,1) -- (31,2) -- (33,2) -- (33,1) -- cycle; 
 \draw[fill=black]  (32,2) -- (32,3) -- (34,3) -- (34,2) -- cycle; 
\draw[-stealth, line width=2mm] (37,3) -- (38,3) node[anchor=west] {};

\draw[] (40,1) -- (40,5); \draw[dashed] (40.5,0) -- (40.5,6); 
\draw[] (41,1) -- (41,5);\draw[dashed] (41.5,0) -- (41.5,6); 
 \draw[] (42,1) -- (42,5); \draw[dashed] (42.5,0) -- (42.5,6); 
 \draw[] (43,1) -- (43,5);\draw[dashed] (43.5,0) -- (43.5,6); 
\draw[] (44,1) -- (44,5);\draw[dashed] (44.5,0) -- (44.5,6); 
 \draw[] (45,1) -- (45,5);    
 \draw[] (40,1) -- (45,1);  \draw[dashed] (39,1.5) -- (46,1.5);   
\draw[] (40,2) -- (45,2);  \draw[dashed] (39,2.5) -- (46,2.5);   
 \draw[] (40,3) -- (45,3); \draw[dashed] (39,3.5) -- (46,3.5);   
 \draw[] (40,4) -- (45,4);\draw[dashed] (39,4.5) -- (46,4.5);   
 \draw[] (40,5) -- (45,5);  
%\draw[-stealth, line width=2mm] (47,3) -- (48,3) node[anchor=west] {};
\draw[fill=black]  (42,3) -- (42,4) -- (43,4) -- (43,3) -- cycle;
\draw[fill=black]  (43,2) -- (43,3) -- (44,3) -- (44,2) -- cycle;
\draw[fill=black]  (41,1) -- (41,2) -- (44,2) -- (44,1) -- cycle;
\end{tikzpicture}
}
\caption{Steps of the evolution of the Game of Life. Black cells are live; white cells are dead.}
    \label{fig:evolution}
\end{figure}

%We proceed to formally define CA. 

\begin{defn}[Cellular automaton]\label{defn:CA}
Let $d, n,k \in \mathbb{N}, k\geq 2$. A cellular automaton is an ordered quadruple $(\mathbb{Z}^d, K, \mathcal{E}, f)$, where \begin{enumerate}
    \item the $d$-dimensional lattice $\mathbb{Z}^d$ is referred to as the cellular space,
    \item $K = \{0,\frac{1}{k-1},\ldots \hspace{-0.02cm}, \frac{k-2}{k-1}, 1\}$ is the state set with $k$ states,
    \item $\mathcal{E}  = \{0_d, z_1,\ldots \hspace{-0.02cm}, z_{n-1}\}\subset \mathbb{Z}^d$ is the neighborhood set,
    \item $f: K^n \rightarrow K $ is the transition function.
\end{enumerate} 
The lattice dimension $d$ is called the cellular space dimension. Cells are centered on the lattice points of $\mathbb{Z}^d$. A configuration, denoted by $c$, over the cellular space assigns a state in the state set to each cell (or equivalently lattice point). For $z \in \mathbb{Z}^d$, denote by $c[z]$ its state under the configuration~$c$. Let $\Ccal$ be the set of all possible configurations over the cellular space. The CA map $F: \Ccal \rightarrow \Ccal$ effecting the configuration evolution from one time step to the next is determined by the transition function $f$ according to \begin{equation}\label{eq:globalmap}
     F(c)[z] = f(c[z], c[z+z_1],\ldots \hspace{-0.02cm}, c[z+z_{n-1}]), \quad \text{ for all } z\in \mathbb{Z}^d.
\end{equation} 
%For a given cell $z\in \mathbb{Z}^d$, it has neighbors $\{z+z_1,\ldots \hspace{-0.02cm}, z+z_{n-1}\}$ and the associated 
%neighborhood $\{z,z+z_1,\ldots \hspace{-0.02cm}, z+z_{n-1}\}$. 
\end{defn}

\begin{rem}
The CA map $F$ is often referred to as the ``global mapping function'', whereas the transition function $f$ is called the ``local mapping function''. 
%takes a given configuration to a new configuration by assigning to each cell in the cellular space a new state according to the transition rule $f$. The CA %map is therefore often also referred to as the ``global mapping function'', whereas the transition rule is called the ``local mapping function''. 
Note that our choice of the state set $K$ in \defref{defn:CA} is without loss of generality, as for every general state set $K^*$, there is a bijection $h: K^* \rightarrow K$, with $|K^*| = |K|$. For example, for the state set $K^* = \{1,2,3\}$, $h: K^* \rightarrow \{0,1/2,1\}, h(x) = (x-1)/2$ is a bijection. 
\end{rem}

We continue with examples illustrating \defref{defn:CA}.
\begin{exmp}
Consider a CA with cellular space dimension $d =1$, neighborhood set $\mathcal{E}=\{0,-1\}\subset \mathbb{Z}$, state set $K = \{0,1\}$, and the transition function $f:K^2\rightarrow K$ as provided in \tblref{tbl:example1}. The CA map $F$ is given by 
\begin{equation*}
     F(c)[z] = f( c[z], c[z-1]), \quad \text{ for all } z\in \mathbb{Z}.
\end{equation*}
\figref{fig:example1} depicts one step of the CA evolution. 
\begin{center}
\vspace{0.5cm}
\begin{tabular}{|c |c |c |c |c |} 
 \hline
$x_0\,x_{-1}$& 00 & 10 & 01 & 11 \\ \hline
$f(x_0, x_{-1})$ & 0 & 0 & 0 & 1   \\ \hline
\end{tabular}
\vspace{-0.2cm}
\captionof{table}{Transition function $f$.}\label{tbl:example1}
\end{center}
\begin{figure}[H]
\centering
\resizebox{4cm}{!}{

\begin{tikzpicture}
\draw (-4,1) -- (4,1) node[anchor= east] {};
%\draw[dashed] (-4.5,0.5) -- (4.5,0.5) node[anchor= east] {};
\draw (-4,0) -- (4,0) node[anchor= east] {};

\draw (-4,1) -- (-4,0) node[anchor= east] {};
\draw (-3,1) -- (-3,0) node[anchor= east] {};
\draw (-2,1) -- (-2,0) node[anchor= east] {};
\draw (-1,1) -- (-1,0) node[anchor= east] {};
\draw (0,1) -- (0,0) node[anchor= east] {};
\draw (1,1) -- (1,0) node[anchor= east] {};
\draw (2,1) -- (2,0) node[anchor= east] {};
\draw (3,1) -- (3,0) node[anchor= east] {};
\draw (4,1) -- (4,0) node[anchor= east] {};
\draw[fill=black]  (-4,0) -- (-4,1) -- (-3,1) -- (-3,0) -- cycle;
\draw[fill=black]  (-3,0) -- (-3,1) -- (-2,1) -- (-2,0) -- cycle;
\draw[fill=black]  (0,0) -- (0,1) -- (1,1) -- (1,0) -- cycle;
\draw[fill=black]  (2,0) -- (2,1) -- (3,1) -- (3,0) -- cycle;
\node at (-4.5,0.5){ $\cdots$};
\node at (4.5,0.5){ $\cdots$};

\draw[-stealth, line width=1mm] (0,-0.4) -- (0,-1) node[anchor= east] {};

%\node at (0,-0.5){ $\downarrow$};

\draw (-4,-1.4) -- (4,-1.4) node[anchor= east] {};
%\draw[dashed] (-4.5,-1.5) -- (4.5,-1.5) node[anchor= east] {};
\draw (-4,-2.4) -- (4,-2.4) node[anchor= east] {};

\draw (-4,-1.4) -- (-4,-2.4) node[anchor= east] {};
\draw (-3,-1.4) -- (-3,-2.4) node[anchor= east] {};
\draw (-2,-1.4) -- (-2,-2.4) node[anchor= east] {};
\draw (-1,-1.4) -- (-1,-2.4) node[anchor= east] {};
\draw (0,-1.4) -- (0,-2.4) node[anchor= east] {};
\draw (1,-1.4) -- (1,-2.4) node[anchor= east] {};
\draw (2,-1.4) -- (2,-2.4) node[anchor= east] {};
\draw (3,-1.4) -- (3,-2.4) node[anchor= east] {};
\draw (4,-1.4) -- (4,-2.4) node[anchor= east] {};

\draw[fill=black]  (-3,-2.4) -- (-3,-1.4) -- (-2,-1.4) -- (-2,-2.4) -- cycle;

\node at (-4.5,-1.9){ $\cdots$};
\node at (4.5,-1.9){ $\cdots$};

%\node at (sum.center){ $+$};

%h3

% add vertical dashed lines
\draw[dashed] (-2.5,-0.3) -- (-2.5,1.3) node[anchor= east] {};
\draw[dashed] (-3.5,-0.3) -- (-3.5,1.3) node[anchor= east] {};
\draw[dashed] (-1.5,-0.3) -- (-1.5,1.3) node[anchor= east] {};
\draw[dashed] (-0.5,-0.3) -- (-0.5,1.3) node[anchor= east] {};
\draw[dashed] (0.5,-0.3) -- (0.5,1.3) node[anchor= east] {};
\draw[dashed] (1.5,-0.3) -- (1.5,1.3) node[anchor= east] {};
\draw[dashed] (2.5,-0.3) -- (2.5,1.3) node[anchor= east] {};
\draw[dashed] (3.5,-0.3) -- (3.5,1.3) node[anchor= east] {};
%arrows

\draw[dashed] (3.5,-2.7) -- (3.5,-1.1) node[anchor= east] {};
\draw[dashed] (2.5,-2.7) -- (2.5,-1.1) node[anchor= east] {};
\draw[dashed] (1.5,-2.7) -- (1.5,-1.1) node[anchor= east] {};
\draw[dashed] (0.5,-2.7) -- (0.5,-1.1) node[anchor= east] {};
\draw[dashed] (-0.5,-2.7) -- (-0.5,-1.1) node[anchor= east] {};
\draw[dashed] (-1.5,-2.7) -- (-1.5,-1.1) node[anchor= east] {};
\draw[dashed] (-2.5,-2.7) -- (-2.5,-1.1) node[anchor= east] {};
\draw[dashed] (-3.5,-2.7) -- (-3.5,-1.1) node[anchor= east] {};
\end{tikzpicture}

}
\caption{The state $1$ is indicated by black cells, the state~$0$ by white cells.}\label{fig:example1}
\end{figure}
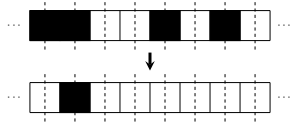
\end{exmp}

\begin{exmp}[Game of Life]
Consider a CA with cellular space dimension $d =2$, state set $K = \{0,1\}$, and neighborhood set
$\mathcal{E} = \{0_2,z_1,\ldots \hspace{-0.02cm},\\z_8\} \subset \mathbb{Z}^2$, where 
\begin{equation*}
\begin{aligned}
     z_1 &= (-1,1), \, z_2 = (0,1), \, z_3 = (1,1), \, z_4=(-1,0), \\[1mm]
    z_5 &= (1,0), \, z_6 = (-1,-1), \, z_7 = (0,-1), \, z_8=(1,-1).
\end{aligned}
\end{equation*}
The transition function $f:K^9\rightarrow K$ is given by 
\begin{equation*}
    f(x_0,\ldots \hspace{-0.02cm},x_8) =\begin{cases}
      \mathbbm{1}_{x_1+\cdots+x_8=3}   , &\text{if } x_0=0,\\
      \mathbbm{1}_{x_1+\cdots+x_8=2}+\mathbbm{1}_{x_1+\cdots+x_8=3},    & \text{if } x_0=1.
    \end{cases}
\end{equation*} 
The resulting CA is the Game of Life with dead cells corresponding to state $0$ and live cells to state $1$.  
\end{exmp}

One of the main goals of this paper is to show that deep ReLU networks are capable of learning CA transition functions from CA evolution sequences. This result also establishes that every CA transition function can be realized by a deep ReLU network. Concretely, we build this universal CA transition function representation theorem in two steps. First, we establish a correspondence between CA transition functions and MV algebras \cite{cignoli2013algebraic}. Then, we show that 
statements in MV logic are naturally expressed by deep ReLU networks. En route it is shown that MV logic is a suitable language for describing CA with general state sets as logical machines.

%The connection between MV logic and CA is interesting in its own right as
%it generalizes the connection between binary CA and Boolean logic.
%elucidates the 
%technique we develop to prove all this is interesting in its own right as it elucidates the role of MV algebras as bridges between CA and neural networks.

\subsection{One-dimensional binary CA}\label{sec:1.2Onedim}

In order to develop intuition on the connection between CA, MV logic, and ReLU networks, we consider a simple example setup. Specifically, we investigate CA with cellular space~$\mathbb{Z}$, state set $K = \{0,1 \}$, and neighborhood set $\mathcal{E} = \{-1,0,1\}$. This class of CA is referred to as elementary CA \cite{wolfram1983statistical}. As the transition function depends on the states of the cell itself and its left and right neighbor, there are $2^3$ possible states for a given cell's neighborhood and hence a total of $2^{2^3} = 256$ transition functions. Wolfram proposed a scheme \cite{wolfram1983statistical} to index the associated CA using an $8$-digit binary string. For example, \tblref{tbl:truth30} specifies the transition function of the elementary CA of index 30 ($30_{10} = 0001 1110_{2}$), and \tblref{tbl:truth110} contains that corresponding to elementary CA 110 ($110_{10} = 0110 1110_{2}$). It is remarkable that elementary CA, despite their simplicity, are capable of universal computation. Specifically, 
%can function \textcolor{blue}{Despite being one of the simplest classes among all possible CA, elementary CA already can function as models of computation. 
Cook \cite{cook2004universality} proved that elementary CA 110 is Turing complete.
\begin{center}
\vspace{0.5cm}
\begin{tabular}{|c |c |c |c |c |c |c |c |c |} 
 \hline
$x_{-1}\,x_0\,x_1$& 111 & 110 & 101 & 100 & 011 & 010 & 001 & 000 \\ \hline
$f_{30}(x_{-1}, x_0, x_1)$ & 0 & 0 & 0 & 1 & 1 & 1 & 1 & 0  \\ \hline
\end{tabular}
\vspace{-0.2cm}
\captionof{table}{Transition function of elementary CA 30.} \label{tbl:truth30} 
\end{center}

\begin{center}
\vspace{0.3cm}
\begin{tabular}{|c |c |c |c |c |c |c |c |c |} 
 \hline
$x_{-1}\,x_0\,x_1$& 111 & 110 & 101 & 100 & 011 & 010 & 001 & 000 \\ \hline
$f_{110}(x_{-1}, x_0, x_1)$ & 0 & 1 & 1 & 0 & 1 & 1 & 1 & 0  \\ \hline
\end{tabular}
\vspace{-0.2cm}
\captionof{table}{Transition function of elementary CA 110.} \label{tbl:truth110} 
\end{center}

We now bring Boolean logic into the picture, concretely by interpreting the state set elements 0 and 1 as truth values in Boolean algebra, and assuming that $x_{-1}, x_0, x_1$ are Boolean propositional variables. 
The transition function $f_{110}$, for example, can in effect be regarded as a Boolean function, namely $f_{110} = \texttt{OR}(\texttt{XOR}(x_0,x_1), \texttt{AND}(\texttt{NOT}(x_{-1}),\texttt{OR}(x_0,x_1)))$.
%The transition rule $f_{254}$, for example, can hence in effect be regarded as a Boolean function realizing the \texttt{OR} operation, $f_{254}(x_{-1}, x_0, x_1) = \texttt{OR}(x_{-1}, x_0, x_1)$. 
In the same manner, we get $f_{30}(x_{-1}, x_0, x_1) = \texttt{XOR}(x_{-1}, \texttt{OR}(x_0,x_1))$. As shown in \cite[pp.~29-31]{voorhees1996computational}, \cite[Section~12.2]{rosen2007discrete}, the transition function of every elementary CA can be expressed as a formula in Boolean logic, particularly in a canonical form called disjunctive normal form (DNF), which we define next. 

\begin{defn}\label{defn:DisjunctiveNormalForm}
In Boolean logic, a finite string of symbols is called a \textit{conjunction} if it is given by \texttt{AND} connections of a finite number of propositional variables (possibly negated), and a \textit{disjunction} if the connections are effected through the \texttt{OR} operation. A logical formula is said to be in \textit{disjunctive normal form} if it is given by the \texttt{OR} connection of a finite number of conjunctions.
\end{defn}

Let the $\texttt{OR}$ operation be denoted by $\oplus$, $\texttt{AND}$ by $\odot$, and $\texttt{NOT}$ by $\lnot$. Following the procedure in \cite[Section~12.2]{rosen2007discrete}, we can express the transition function in \tblref{tbl:truth30} in DNF according to
\begin{equation}\label{eq:rule30logic}
f_{30} = (x_{-1}\odot  \lnot x_0 \odot \lnot x_1) \oplus (\lnot x_{-1}\odot   x_1 ) \oplus (\lnot x_{-1}\odot   x_0).
\end{equation}

Having established the connection between elementary CA and Boolean logic, by expressing the transition function in terms of a Boolean formula, we are ready to have ReLU networks enter the story. Specifically, this shall be effected by realizing the Boolean logical operations $\oplus, \odot,$ and~$\lnot$ through ReLU networks (see \defref{defn:ReLU} below). With the ReLU function $\rho: \mathbb{R} \rightarrow \mathbb{R}, \hspace{0.2cm} \rho (x):=\max\{0,x\}$, the binary and $n$-ary conjunction can be written as 
\begin{align}
x_1 \odot x_2 &= \rho\hspace{0.02cm}(x_1 + x_2 - 1), \\\label{eq:ReLUandgate}
x_1 \odot \cdots \odot x_n &=\rho\hspace{0.02cm}(\Sigma_{i = 1}^n x_i - (n-1)).
\end{align}
Likewise, the binary and $n$-ary disjunction can be realized according to
\begin{align}
x_1 \oplus x_2 &=\rho\hspace{0.02cm}(x_1+x_2) -\rho\hspace{0.02cm}(x_1 + x_2 - 1), \\\label{eq:ReLUorgate}
x_1 \oplus \cdots \oplus x_n &=\rho\hspace{0.02cm}(\Sigma_{i = 1}^n x_i) -\rho(\Sigma_{i = 1}^n x_i-1).
\end{align}
Finally, the Boolean operation $\lnot$ can be expressed as 
\[
\lnot x = 1-x.
\]
We now combine these results to realize the logical formula $f_{30}$ in~\eqref{eq:rule30logic} in terms of the ReLU network  $\Phi^{f_{30}}:\mathbb{R}^3 \rightarrow \mathbb{R} $ with three layers according to
\begin{equation}\label{eq:NNf30}
    \Phi^{f_{30}}:= W_3^{f_{30}} \circ \rho \circ W_2^{f_{30}} \circ \rho \circ  W_1^{f_{30}}
\end{equation} where 
\begin{equation*}
 W_1^{f_{30}}(x_{-1},x_0,x_1) = \begin{pmatrix}
    -1 & -1 & 1 \\
    1 & 0 & -1 \\
    0 & 1 & -1 
    \end{pmatrix} \begin{pmatrix}
        x_1 \\x_0 \\x_{-1}
    \end{pmatrix} \text{, }
\end{equation*} 
\begin{equation*}
     W_2^{f_{30}}(x) = \begin{pmatrix}
    1 & 1 & 1  \\
    1 & 1 & 1 
    \end{pmatrix} x + \begin{pmatrix}
   0 \\-1
    \end{pmatrix}\text{, } \quad \text{ for } x\in \mathbb{R}^3,
\end{equation*}
\begin{equation*}
    W_3^{f_{30}}(x) = \begin{pmatrix}1 & -1 \end{pmatrix} x,  \quad \text{ for } x\in \mathbb{R}^2.
\end{equation*}
It can now be verified directly that
\begin{equation*}
     \Phi^{f_{30}}(x_{-1},x_0,x_1) = f_{30}(x_{-1},x_0,x_1), \text{ for } (x_{-1},x_0,x_1 ) \in \{0,1\}^3.
\end{equation*}   

We hasten to add that ReLU network realizations of Boolean formulae are not unique. For example, upon noting that $\Phi^{\text{Id}}(x) = \rho(x) - \rho(-x)=x, x\in \mathbb{R},$ is a valid ReLU network, the map $\Phi^{f_{30}}$ in~\eqref{eq:NNf30} can be augmented according to 
\begin{equation*}
\Phi^{\text{Id}} \circ \cdots \circ \Phi^{\text{Id}} \circ \Phi^{f_{30}}
\end{equation*}
without changing the input-output function it realizes. This results in an infinite collection of different ReLU networks, all realizing the Boolean formula $f_{30}$. 
%%%Y: please add a reference for the statement in the next sentence
In addition, the Boolean logical formula corresponding to a given binary truth table is not unique either \cite[Section~12.4]{rosen2007discrete}. These issues of nonuniqueness will be addressed in more detail in~\secref{sec:Extract}. 

We can summarize what was done in this section as follows. In an effort to implement the transition function of a one-dimensional binary CA by a ReLU neural network, we first expressed the transition function, originally specified in the form of a truth table, by a formula in Boolean logic. This formula was then realized by a ReLU network. The bridge via Boolean algebra effected in this manner serves to illustrate that CA are actually ``logical machines'', an aspect elaborated on in the literature before  \cite{voorhees1996computational}, \cite{wolfram1983statistical}, \cite{von2017general}, but notably only in the binary case. 
%%%Y: do you agree with the ... only in the binary case ...

Inspired by this connection, we shall show that the transition functions of general CA with arbitrary state sets, in particular of arbitrary cardinality, can be expressed in terms of formulae in MV logic \cite{cignoli2013algebraic}, which, in turn, are found to be naturally realized by ReLU networks. Embedding these (feedforward) networks into a recurrent neural network structure \cite{elman1990finding}, we get a dynamical system realization of general CA, entirely in terms of neural networks. 
%%%Y: please replace hutter2022metric by a more standard reference to RNNs
%but we restrict ourselves to the one-dimensional case for ease of exposition.
%econd, it serves to pave the way for establishing a novel connection between CA with state sets of cardinality larger than two and MV %logic \cite{cignoli2013algebraic}. In particular, 

%Second, it serves to pave the way for establishing a novel connection between CA with state sets of cardinality larger than two and MV %logic \cite{cignoli2013algebraic}. In particular, we shall show that the transition rules of general CA can be expressed in terms of MV logical formulae, %which are, in turn, realizable by ReLU networks. Embedding these (feedforward) networks into a recurrent neural network structure \cite{hutter2022metric}, we %get a dynamical system realization of general (one-dimensional) CA, entirely in terms of neural networks.

\subsection{Previous work}

Despite the fact that neural networks and CA stand for parallel efforts in building computational models of the human brain \cite{burks1969neumann}, \cite{kari2005theory}, a profound theoretical understanding of the connections between the two structures seems to be lacking. Wulff and Hertz \cite{wulff_learning_1992} consider shallow neural networks with the threshold activation function to learn the transition functions of one- and two-dimensional binary CA. While no details on the training algorithm are provided, the experiments in~\cite{wulff_learning_1992} lead to the conclusion that the approach employed is capable of learning only a small subset of the CA considered. More recently, Gilpin~\cite{gilpin2019cellular} designed a convolutional neural network simulating two-dimensional CA with arbitrary state sets and neighborhood size.

The connection between general CA and MV logic we report here seems completely new. For binary CA and Boolean logic, besides the classical results in \cite{voorhees1996computational}, \cite{wolfram1983statistical}, more recent work studying CA evolution through the lens of formal logic is reported in \cite{das2021formal}, \cite{ishida2014cellular}. Specifically, Sukanta and Chakraborty \cite{das2021formal} show that binary CA on finite one-dimensional cellular spaces are models of certain logical languages in the domain of binary strings. Inokuch et al. \cite{ishida2014cellular} use the interpretation of binary CA transition functions as Boolean logical formulae to demonstrate that the multiplication of logical formulae, defined by monoid action, corresponds to the composition of global CA maps associated with the logical formulae under consideration.  

As for the connection between MV logic and neural networks, Amato et al. \cite{amato2002neural} establish a correspondence between formulae in 
{\L}ukasiewicz logic and neural networks with rational weights and the clipped ReLU (CReLU) nonlinearity $\sigma(x) = \min\{1, \max\{0,x\}\}$. In \cite{di2013adding} this result is further extended to a correspondence between Riesz MV algebras and CReLU networks with real weights; this is accomplished by adding scalar multiplication over the real numbers to \Luka propositional logic.

\subsection{Outline of the paper}
\secref{sec:MVLogic} establishes that CA are in essence machines performing MV logic.
%the connection between CA and MV logic. 
%Many-valued algebra is introduced and reviewed starting from first principles. 
In~\secref{sec:NetworkConstruction}, we show how formulae in MV logic can be realized by (deep) ReLU networks.
%provide constructions of deep ReLU networks that realize CA transition rule and one-dimensional CA map, establishing a systematic connection between CA and %neural networks. 
The problem of extracting CA transition functions, along with underlying formulae in MV logic, from ReLU networks that have been trained on corresponding CA evolution data is addressed
in~\secref{sec:Extract}. 
%We conclude where we also propose a procedure for reading out CA transition rules from trained networks.
%~\secref{sec:Conclusion} presents our concluding remarks. 

%%%%%%%%%%%%%%%%%%%%%%%%%%%%%%%%%%%%%%%%%%%%%%%%%%%%%%%%%%%%%%%%%%%%%%%%%%%%%%%%%
%%%%%%%%%%%%%%%%%%%%%%%%%%%%%%%%%%%%%%%%%%%%%%%%%%%%%%%%%%%%%%%%%%%%%%%%%%%%%%%%%
%%%%%%%%%%%%%%%%%%%%%%%%%%%%% section 2 %%%%%%%%%%%%%%%%%%%%%%%%%%%%%%%%%%%%%%%%%
\section{Many-valued logic and cellular automata}\label{sec:MVLogic}
We now turn to generalizing the correspondence between one-dimensional binary CA and Boolean logic. Specifically, we show that CA with arbitrary cellular space dimension and arbitrary state set (cardinality) can be seen as machines carrying out operations in \Luka propositional logic \cite{tarski1983logic}, a many-valued extension of Boolean logic. The corresponding algebraic counterpart is known as Chang's many-valued algebras \cite{chang1958algebraic}. The connection we uncover is built on a fundamental result in MV logic,  which states that every piecewise linear function, whose linear pieces have integer coefficients, corresponds to a formula in MV logic and vice versa. Known in the literature as McNaughton theorem \cite{mcnaughton1951theorem}, this result constitutes the equivalent of translating between binary truth tables and Boolean logical formulae.

\subsection{Many-valued algebras}
We start with a brief review, following~\cite{cignoli2013algebraic}, of the basic elements in the theory of MV algebras.

\begin{defn}\label{defn:MV algebra}
A many-valued algebra is a structure $\mathcal{A} = \langle A, \oplus, \lnot,0\rangle$ consisting of a nonempty set $A$, a constant $0 \in A$, a binary operation  $\oplus$, and a unary operation $\lnot$ satisfying the following axioms:  
\begin{subequations}
\renewcommand{\theequation}{\theparentequation.\arabic{equation}}
\begin{align}
 \label{eq:associativity}    x \oplus (y \oplus z) &= (x \oplus y) \oplus z\\ 
  \label{commutivity}   x\oplus y  &= y\oplus x\\
 \label{eq:NeutralZero}    x \oplus 0  &=x\\
   \label{eq:NOTNOT}  \lnot \lnot x  &= x\\
 \label{eq:NeutralOne}    x \oplus \lnot 0  &= \lnot 0\\
  \label{eq:wedge operations}   \lnot (\lnot x \oplus y)\oplus y  &= \lnot (\lnot y \oplus x)\oplus x.
\end{align}
\end{subequations}
\end{defn}
Specifically,~\eqref{eq:associativity}-\eqref{eq:NeutralZero} state that the structure $\langle A, \oplus, 0 \rangle$ is an abelian monoid. An MV algebra $\langle A, \oplus, \lnot,0\rangle$ is said to be nontrivial iff $|A| > 1$. On each MV algebra we can define a constant $1$ and a binary operation~$\odot$ as follows: \begin{align}
  \label{eq:1}    1 &:=\lnot 0  \\
  \label{eq:odot}    x \odot y&:= \lnot (\lnot x \oplus \lnot y).
\end{align} The ensuing identities are then direct consequences of~\defref{defn:MV algebra}:
\begin{subequations}
\renewcommand{\theequation}{\theparentequation.\arabic{equation}}
\begin{align}
  \label{eq:associativity for times} x \odot (y \odot z) &= (x \odot y) \odot z\\ 
    x\odot y  &= y\odot x\\
   x \odot 1  &=x\\     
  \label{eq:0 for times}   x \odot  0  &=  0.
\end{align}
\end{subequations}

We will frequently use the notions of MV term and term function formalized as follows.

%\begin{defn}[String] 
%A finite sequence of elements of the nonempty set $S$ is called a string over $S$. 
%\end{defn}
\begin{defn}[MV term]\label{defn:MV terms}
Let $n\in \mathbb{N}$ and $S_n = \{(, ), 0,\lnot, \oplus, x_1,\ldots \hspace{-0.02cm},$ $x_n\}$. An MV term in the variables $x_1, \ldots \hspace{-0.02cm}, x_n$ is a string over $S_n$ arising from a finite number of applications of the operations $\lnot$ and $\oplus$ as follows. The elements $0$ and $x_i$, for $i=1,\ldots \hspace{-0.02cm},n$, considered as one-element strings, are MV terms.
\begin{enumerate}
  \item If the string $\tau$ is an MV term, then $\lnot \tau $ is also an MV term. 
  \item If the strings $\tau$ and $\gamma $ are MV terms, then $(\tau \oplus \gamma) $ is also an MV term.
\end{enumerate}
In the remainder of the paper, we write $\tau (x_1, \ldots \hspace{-0.02cm}, x_n)$ to emphasize that~$\tau$ is an MV term in the variables $x_1, \ldots \hspace{-0.02cm},x_n$.
\end{defn}
For instance, the following strings over $S_2 = \{(, ), 0,\lnot, \oplus, x_1,x_2\}$ are MV terms in the variables $x_1$ and $x_2$:
\[
0, x_1, x_2,\lnot 0, \lnot x_2, x_1\oplus \lnot x_2.
\]
%We shall omit the outermost pair of brackets for conciseness, i.e., we write $ x_1\oplus \lnot x_2$ instead of $(x_1\oplus \lnot x_2)$. 
%Besides, for brevity we use the symbols $\odot$ and $1$ as abbreviations according to \eqref{eq:1} and \eqref{eq:odot} when writing MV terms. 

MV terms are syntactic expressions without semantic meaning. To endow them with semantics, an underlying MV algebra must be specified.
%actually logical formulae without operational meaning. To endow them with meaning, an MV algebra must be specified. 
The resulting functions will be referred to as term functions.
%which we define presently, are obtained by interpreting the operations $\oplus$ and $\lnot$ according to how they are specified in the MV algebra.

\begin{defn}[Term function]\label{defn:TermFunction}
Let $\tau(x_1, \ldots \hspace{-0.02cm}, x_n)$  be an MV term and
%in the variables $x_1, \ldots \hspace{-0.02cm}, x_n$. 
%Let 
$\mathcal{A} = \langle A, \oplus, \lnot,0\rangle$ an MV algebra. The term function $\tau^\mathcal{A}:A^n\rightarrow A $ associated with $\tau$ 
under $\mathcal{A}$ is obtained by substituting, for $i=1,\ldots \hspace{-0.02cm}, n$, $a_i \in A$ for all occurrences of $x_i$ in $\tau$ and
interpreting the symbols $\oplus$ and $\lnot$ according to how they are specified in $\mathcal{A}$.
\end{defn}

\iffalse 
defined as follows. For every input $(a_1,\ldots \hspace{-0.02cm}, a_n) \in A^n$, first substitute $a_i$ for all occurrences of the variable $x_i$ in $\tau$:
\begin{enumerate}
\item $x_i^\Acal: (a_1,\ldots \hspace{-0.02cm}, a_n)\mapsto a_i$,\quad for $i=1,\ldots \hspace{-0.02cm}, n$.
\end{enumerate} 
Then, proceed by induction on the number of operation symbols, i.e., $\oplus$ and $\lnot$, occuring in $\tau$ by applying the following rules:
\begin{enumerate}[resume]
\item $(\lnot \tau)^\Acal = \lnot \tau^\Acal$,\quad for MV term $\tau$,
\item $ (\tau \oplus \gamma )^\Acal = \tau ^\Acal \oplus \gamma ^\Acal$, \quad for MV terms $\tau$ and $\gamma$.
\end{enumerate} 
\fi 

We next demonstrate that Boolean algebra is an MV algebra.
\begin{defn}[Boolean algebra]\label{defn:boolean}
Consider the set $B = \{0,1\}.$ Define the binary operation $\oplus$ and the unary operation $\lnot$ on $B$ according to
\begin{equation*}
\begin{aligned}[c]
0\oplus 0 &=0\\
0\oplus 1 &=1\\
1\oplus 0 &=1 \\
1\oplus 1 &=1
\end{aligned}
\qquad \text{and} \qquad
\begin{aligned}[c]
\lnot 0 &=1\\
\lnot 1 &=0.\\
\end{aligned}
\end{equation*} It is readily verified that the resulting structure $ \mathcal{B} = \langle B, \oplus, \lnot  ,0\rangle$ satisfies the axioms in~\defref{defn:MV algebra} and is hence an MV algebra. In fact, by further defining the binary operation $\odot$ on $B$ through
\begin{equation*}
  x  \odot y := \lnot(\lnot x \oplus \lnot y)\text{,}
\end{equation*} which yields \begin{equation*}
    \begin{aligned}
        0 \odot 0 &= 0 \\
    0 \odot 1 &= 0 \\
    1 \odot 0 &= 0 \\
    1 \odot 1 &= 1, 
    \end{aligned} 
\end{equation*} we immediately see that the operations $\odot, \oplus$, and $ \lnot$ correspond to the \texttt{AND}, \texttt{OR}, and \texttt{NOT} operations, respectively, in Boolean logic with $0$ designating \texttt{False} and $1$ standing for \texttt{True}.
\end{defn}

Recall the informal discussion in Section~\ref{sec:1.2Onedim} explaining how the transition functions of binary one-dimensional CA can be expressed in terms of Boolean logic. We now formalize this observation (for arbitrary cellular space dimension) by casting it into MV algebras.
%and en passant extending it to multi-dimensional cellular spaces. 

\begin{lem}\label{lemma2.6} 
Consider a CA with cellular space dimension $d \in \mathbb{N}$, neighborhood size $n \in \mathbb{N}$, state set $K = \{0,1\}$, and transition function $f:K^n \rightarrow K$ specified in the form of a (Boolean) truth table. There exists an MV term $\tau(x_1,\ldots \hspace{-0.02cm}, x_n)$ with associated term under the Boolean algebra $\mathcal{B} =  \langle  \{0,1\}, \oplus, \lnot  ,0\rangle$ satisfying \begin{equation*}
    \tau^{\mathcal{B}}(x_1,\ldots \hspace{-0.02cm}, x_n) = f(x_1,\ldots \hspace{-0.02cm}, x_n), \quad \text{ for all } (x_1,\ldots \hspace{-0.02cm}, x_n) \in \{0,1\}^n.
\end{equation*}  
\end{lem}

\begin{proof}
We first transform the truth table specifying the transition function $f$
%in terms of a truth table 
%binary CA transition rule $f$ is either already specified in terms of a Boolean function or else provided in the form of a truth table, which
%can be 
into a Boolean formula \cite[Section~12.2]{rosen2007discrete}.
%%%Y: Why do we need to transform into DNF?
%In both cases, it can be transformed into DNF \cite[Section 12.2]{rosen2007discrete}.
Identifying, in the resulting Boolean algebraic expression,
%which following \defref{defn:TermFunction} and
%identifying 
the Boolean operations 
$\texttt{AND}$, $\texttt{OR}$, and $\texttt{NOT}$ with the operations 
$\odot$, $\oplus$, and $\lnot$ in a general MV algebra yields a valid MV term
$\tau$. Converting $\tau$ into its associated term function~$\tau^{\mathcal{B}}$ according to~\defref{defn:TermFunction} establishes the result. 
%consisting of the Boolean operations \texttt{AND}, \texttt{OR}, and $\texttt{NOT}$ 
\end{proof}

In order to extend~\lemref{lemma2.6} to CA with state sets of arbitrary cardinality, we need a logic that can cope with more than two truth values. We shall see that the theory of MV logic \cite{cignoli2013algebraic} provides a suitable framework and start by developing the underlying algebraic structure, namely that of MV algebras.

\subsection{MV algebras and CA}\label{2.2}
We begin with a simple example of an MV algebra. 
\begin{exmp}\label{exmp:MVAlgebraAk}
For $k\in\mathbb{N}$, $k \geq 2$, consider the CA state set \begin{equation}\label{eq:KStateSet}
    K = \Bigl\{  0, \frac{1}{k-1}, \ldots \hspace{-0.02cm}, \frac{k-2}{k-1},1 \Bigr\}
\end{equation} of cardinality $k$. Define the binary operation $\oplus$ and the unary operation~$\lnot$ on $K$ according to
\begin{align*}
    x \oplus y &= \min \{1, x+y\} \\
    \lnot x &= 1-x.
\end{align*} It is readily seen that the structure $\mathcal{A}_k:=\langle K, \oplus, \lnot,0 \rangle$ satisfies the axioms in~\defref{defn:MV algebra} and hence constitutes an MV algebra. For $k=2$, $\mathcal{A}_k$ reduces to the Boolean algebra in~\defref{defn:boolean}.
\end{exmp}
For the MV algebra $\mathcal{A}_k$ in \exmpref{exmp:MVAlgebraAk}, given an MV term $\tau$ in $n$ variables, the associated term function $\tau^{\mathcal{A}_k}: K^n \rightarrow K $ can be interpreted as a CA transition function. 
We illustrate this observation through a simple example, but first note that $\tau^{\mathcal{A}_k}$ is guaranteed to map
$K^n$ to $K$ owing to the specific structure of the state set $K$ and the operations $\oplus$ and $\lnot$.
%by identifying the MV term underlying two well-known multi-state CA.
%through the following two examples of multi-state CA. 

\begin{exmp}\label{exmp:2.1}
Consider the MV algebra $\mathcal{A}_3$ according to Example~\ref{exmp:MVAlgebraAk} with the associated set $K = \{0, 1/2,1\}$.
%with
%Let $k=3$ in~\eqref{eq:KStateSet} with 
%the corresponding set $K = \{0, 1/2,1\}$. 
The term function $\tau^{ \mathcal{A}_3}: K^3 \rightarrow K$ corresponding to the 
MV term $\tau = x_{-1}\oplus x_0 \oplus x_1$ equals the transition function of the $3$-color totalistic CA \cite{wolfram1983statistical} with cellular space $\mathbb{Z}$, neighborhood set $\mathcal{E}=\{-1, 0, 1\}$, state set $K$, and transition function $f$ as specified in~\tblref{tbl:totalistic} (note that $f$ depends only on the sum of the neighborhood state values).

\begin{center}
\vspace{0.5cm}
\begin{tabular}{|c |c |c |c |c |c |c |c |c |} 
 \hline
$x_{-1}+x_0+x_1$& 3 & 5/2 & 2 & 3/2 & 1 & 1/2 & 0  \\ \hline
$f(x_{-1}, x_0, x_1)$ & 1 & 1& 1 & 1 & 1 & 1/2 & 0   \\ \hline
\end{tabular}
\vspace{-0.2cm}
\captionof{table}{Transition function $f$ of the totalistic CA.} \label{tbl:totalistic} 
\end{center}
\end{exmp}

Inspired by this example, we now ask whether every CA transition function has an underlying MV term with associated term function under a suitable MV algebra equal to the transition function. Formally speaking, we seek a generalization of Lemma~\ref{lemma2.6} to the case of CA with arbitrary state set cardinality $k$. If answered in the affirmative, this would show that CA are essentially machines that perform operations in a suitable logic. It turns out that this is, indeed, the case, but we need to consider an MV algebra that accommodates state sets $K$ of arbitrary cardinality~$k$.
%namely the unit interval on $\mathbb{R}$ endowed with truncated addition and negation as formalized next.

\iffalse 
For a CA with the state set $K$ in~\eqref{eq:KStateSet}, neighborhood size $n\in \mathbb{N}$, and arbitrary cellular space dimension $d\in \mathbb{N}$, the associated transition rule is a function $f: K^n \rightarrow K$. Conversely, as illustrated by \exmpref{exmp:2.1} and \exmpref{exmp:2.2}, for the MV algebra $\mathcal{A}_k$ in \exmpref{exmp:MVAlgebraAk}, given an MV term $\tau$ in $n$ variables, the associated term function $\tau^{\mathcal{A}_k}: K^n \rightarrow K $ can be interpreted as a CA transition rule. Based on this simple observation, we will now develop a systematic correspondence between CA and MV term functions. This correspondence endows CA transition rules with logical meaning. In order to elucidate the fundamental structure of the so-identified connection, we will enlarge $\mathcal{A}_k$ to encompass all possible CA state sets. This will require more general underlying MV algebras, which, as it turns out, are given by the unit interval on $\mathbb{R}$ endowed with truncated addition and negation.   
\fi 

\begin{defn}\label{defn:MV01}
Consider the unit interval $[0, 1]$ on $\mathbb{R}$, and define $x \oplus y = \min \{1, x+y\}$ and $\lnot x = 1-x$, for  $x, y \in [0,1]$. It can be verified that the structure $\mathcal{I} = \langle [0,1], \oplus, \lnot,0 \rangle $ is an MV algebra. In particular, $\mathcal{I}$ constitutes the algebraic counterpart of \Luka propositional logic~\cite{chang1958algebraic}. We further define the operation $x  \odot y := \lnot(\lnot x \oplus \lnot y) = \max\{0,x+y-1\}$.
\end{defn}

The MV algebra $\mathcal{I}$ in \defref{defn:MV01} is the so-called standard MV algebra. 
%%%Y: In the next sentence the references should be to the specific theorems/statements in the books/papers.
As shown in \cite[Section~5]{chang1958algebraic}, \cite{chang1959new}, an equation holds in every MV algebra iff it holds in the standard MV algebra $\mathcal{I}$, endowing $\mathcal{I}$ with universality.
\iffalse 
Furthermore, the unit interval $[0,1]$ accommodates all possible CA state set cardinality. Let $\tau(x_1,\ldots \hspace{-0.02cm},x_n)$ be an MV term and $\tau^{\mathcal{I}}:[0,1]^n \rightarrow [0,1]$ the associated term function under $\mathcal{I}$. For every state set cardinality $k \in \mathbb{N}, k\geq 2$ in~\eqref{eq:KStateSet},  the function $f: K^n \rightarrow K$ given by 
\[
f(x_1,\ldots \hspace{-0.02cm},x_n) := \tau^{\mathcal{I}}(x_1,\ldots \hspace{-0.02cm},x_n), \quad \text{ for all } (x_1,\ldots \hspace{-0.02cm},x_n) \in K^n,
\]
is the transition function of some CA with state set $K$ and neighborhood size $n$. In other words, CA can be interpreted as machines carrying out MV logical operations. Conversely, given a CA transition function, we wish to construct an associated logical formula to endow the underlying CA with logical operational meaning. 
\fi 
Next, recall that for the binary state set $\{0,1\}$, in the proof of Lemma~\ref{lemma2.6} we started by converting the binary truth table specifying
the transition function $f$ into a formula in Boolean logic. In MV logic truth tables are given by mappings from $[0,1]^n$ to $[0,1]$.
While in the Boolean case,
%We next note that 
%It is hence natural to ask what the concept corresponding to truth tables is for MV logic. 
%%%Y: can you please insert a reference at the end of this statement for a proof of the statement the next sentence makes.
%To answer this question, we first note that 
%in the Boolean case, 
every truth table can be cast into a formula in Boolean logic \cite[Section 12.2]{rosen2007discrete}, this does not
hold for the MV algebra $\mathcal{I}$.
%where the role of truth tables is played by functions mapping $[0,1]^{n}$ to $[0,1]$.
The McNaughton theorem, stated next, explicitly characterizes the class of functions from $[0,1]^{n}$ to $[0,1]$ that have underlying MV terms $\tau$. Here, ``underlying'' means that the term function associated with $\tau$ under $\mathcal{I}$ equals the function under consideration.
%corresponding MV terms under the MV algebra $\mathcal{I}$.
%This leads us to the 
%In order to extend this approach to CA with state sets of arbitrary cardinality, the 
%McNaughton theorem, which explicitly characterizes the class of functions which have corresponding MV terms under the MV algebra $\mathcal{I}$.

\begin{thm}\label{them:McNaughtonTheorem}[McNaughton theorem \cite{mcnaughton1951theorem}]
Consider the MV algebra $\mathcal{I} = \langle [0,1], \oplus, \lnot,0 \rangle $ in~\defref{defn:MV01}. Let $n\in \mathbb{N}$. For a function $f_{\text{c}}: [0,1]^n \rightarrow [0,1]$ to have an associated MV term $\tau$ such that $\tau ^{\mathcal{I}} = f_{\text{c}}$ on $[0,1]^n$, it is necessary and sufficient that
%~$f_{\text{c}}$ satisfy the following conditions: 
\begin{enumerate}
\item $f_{\text{c}}$ is continuous with respect to the natural topology on $[0,1]^n$,
    \item there exist linear functions $p_1, \ldots \hspace{-0.02cm}, p_\ell$ with integer coefficients, i.e., 
    \begin{equation*}
        p_j(x_1, \ldots \hspace{-0.02cm}, x_{n}) = m_{j1}x_1+\cdots+m_{jn}x_{n}+b_j, \quad  j = 1,\ldots \hspace{-0.02cm},\ell,
    \end{equation*} where $ m_{j1},\ldots \hspace{-0.02cm}, m_{jn},b_j \in \mathbb{Z}$, for $j=1,\ldots \hspace{-0.02cm}, \ell$, such that for every $x \in [0,1]^n$, there is a $j \in \{1,\ldots \hspace{-0.02cm}, \ell\}$ with $ f_{\text{c}}(x) = p_j(x)$.
\end{enumerate} Functions satisfying these conditions are called McNaughton functions. 
\end{thm}

We now demonstrate how the McNaughton theorem can be employed to decide whether a given CA transition function has an underlying MV term. Since term functions under $\mathcal{I}$ are continuous piecewise linear functions according to~\thmref{them:McNaughtonTheorem}, we start by linearly interpolating CA transition functions $f:K^n\rightarrow K$ to continuous piecewise linear functions $f_{\text{c}}:[0,1]^n\rightarrow [0,1]$. The following simple example illustrates the approach we pursue. 

\begin{exmp}\label{example McNaughton is not enough}
Consider a one-dimensional CA with cellular space $\mathbb{Z}$, neighborhood set $\mathcal{E} = \{ 0\}$, state set $K = \{0, 1/3, 2/3,1\}$, and transition function $f$ according to 
    \vspace{0.3cm}
    \begin{center}
\begin{tabular}{|c |c |c |c |c |} 
 \hline
$x$& $0$ & $1/3$ & $2/3$ & $1$  \\ \hline
$f(x)$ & $0$ & $1$ & $1$ & $0$  \\ \hline
\end{tabular}
\end{center}
 \vspace{0.5cm}
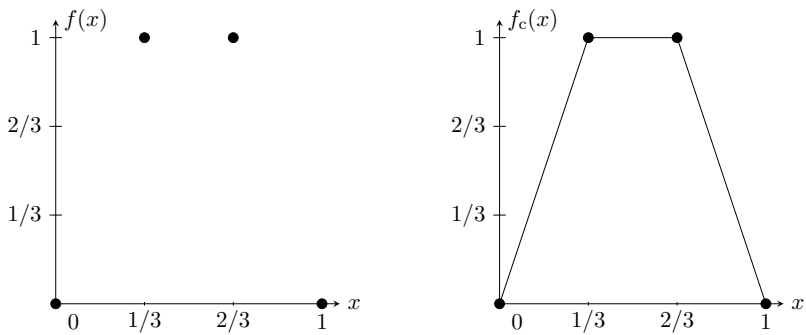
\begin{figure}[h]
\centering
\resizebox{11cm}{!}{

\begin{tikzpicture}[x=1.5cm, y=1.5cm, >=stealth, scale=1, transform shape]
    \draw[->] (0,0) -- (0,3.2) node[anchor=west] {$f(x)$};
    \draw[->] (0,0) -- (3.2,0) node[anchor= west] {$x$};
    \draw (0.07,1) -- (-0.07,1) node[anchor= east] {$1/3$};
     \draw (0.07,2) -- (-0.07,2) node[anchor= east] {$2/3$};
      \draw (0.07,3) -- (-0.07,3) node[anchor= east] {$1$};

    \draw (2,-0.02) -- (2,0.02) node[anchor= north] {$2/3$};
      \draw (1,-0.02) -- (1,0.02) node[anchor= north] {$1/3$};
    \draw (3,-0.02) -- (3,0.02) node[anchor= north] {}; 
  \node at (3,-0.2) {$1$};
    \node at (0.2,-0.3 cm) {$0$};
  \draw[black,fill=black] (0,0) circle (.5ex);
     \draw[black,fill=black] (1,3) circle (.5ex);
     \draw[black,fill=black] (2,3) circle (.5ex);
     \draw[black,fill=black] (3,0) circle (.5ex);

     %%% f'
     \draw[->] (5,0) -- (5,3.2) node[anchor=west] {$f_{\text{c}}(x)$};
    \draw[->] (5,0) -- (8.2,0) node[anchor= west] {$x$};
    \draw (5.07,1) -- (4.93,1) node[anchor= east] {$1/3$};
     \draw (5.07,2) -- (4.93,2) node[anchor= east] {$2/3$};
      \draw (5.07,3) -- (4.93,3) node[anchor= east] {$1$};
    \draw (7,-0.02) -- (7,0.02) node[anchor= north] {$2/3$};
      \draw (6,-0.02) -- (6,0.02) node[anchor= north] {$1/3$};
    \draw (8,-0.02) -- (8,0.02) node[anchor= north] {}; 
    \node at (8,-0.2) {$1$};
    \node at (5.2,-0.3 cm) {$0$};
  \draw[black,fill=black] (5,0) circle (.5ex);
     \draw[black,fill=black] (6,3) circle (.5ex);
     \draw[black,fill=black] (7,3) circle (.5ex);
     \draw[black,fill=black] (8,0) circle (.5ex); 
      \draw (5,0) -- (6,3) node[anchor= north] {};
    \draw (7,3) -- (6,3) node[anchor= north] {};
    \draw (7,3) -- (8,0) node[anchor= north] {};

\end{tikzpicture}
}
\caption{The transition function $f$ and the associated function $ f_{\text{c}}$ in \exmpref{example McNaughton is not enough}.}\label{fig:examplefx}
\end{figure}

The continuous piecewise linear function $ f_{\text{c}}:[0,1] \rightarrow [0,1]$ obtained by interpolating $f$ is given by
%is table of $f$ can be linearly interpolated to a continuous piecewise linear function $ f_{\text{c}}:[0,1] \rightarrow [0,1]$ with 
\begin{equation}\label{eq:egf}
 f_{\text{c}}(x)=
    \begin{cases}
    3x, &0\leq x< 1/3 \\
    1,&1/3\leq x< 2/3 \\
     -3x+3,&2/3\leq x\leq 1 \\
    \end{cases}.
\end{equation}
As $f_{\text{c}}$ in (\ref{eq:egf}) has exclusively integer coefficients and is piecewise linear and continuous by construction, it
satisfies Conditions 1 and 2 in~\thmref{them:McNaughtonTheorem}. Hence, there is an underlying MV term,
shown in Appendix~\ref{app:A} to be given by \begin{equation*}
   \tau=  (x\oplus x\oplus x) \wedge  \lnot 0 \wedge (x\odot x\odot x),
\end{equation*} 
where, for brevity, we write $x \wedge y = (x\oplus \lnot y)\odot y$. 

However, for the same cellular space and state set, a slightly different transition function $g$, given by
\vspace{0.3cm}\begin{center}
\begin{tabular}{|c |c |c |c |c |} 
 \hline
$x$& $0$ & $1/3$ & $2/3$ & $1$  \\ \hline
$g(x)$ & $0$ & $1/3$ & $1$ & $0$  \\ \hline
\end{tabular}
\end{center}
\vspace{0.5cm}
yields the associated piecewise linear function 
\begin{equation*}
 g_{\text{c}}(x)=
    \begin{cases}
    x, &0\leq x< 1/3 \\
    2x-1/3,&1/3\leq x< 2/3 \\
     -3x+3,&2/3\leq x\leq 1 \\
    \end{cases},
\end{equation*}
which fails to satisfy Condition 2 in~\thmref{them:McNaughtonTheorem}, as one of its linear pieces exhibits a noninteger coefficient. As the McNaughton theorem is iff, we can conclude that $g_{\text{c}}$ does not have an underlying MV term $\tau$ such that~$\tau^{\mathcal{I}} =  g_{\text{c}}$.
\begin{figure}[h]
\centering
\resizebox{11cm}{!}{

\begin{tikzpicture}[x=1.5cm, y=1.5cm, >=stealth, scale=1, transform shape]

     %%% f'
  
%%% g
     \draw[->] (0,-5) -- (0,-1.8) node[anchor=west] {$g(x)$};
    \draw[->] (0,-5) -- (3.2,-5) node[anchor= west] {$x$};
   \draw (0.07,-4) -- (-0.07,-4) node[anchor= east] {$1/3$};
     \draw (0.07,-3) -- (-0.07,-3) node[anchor= east] {$2/3$};
      \draw (0.07,-2) -- (-0.07,-2) node[anchor= east] {$1$};

    \draw (2,-5.02) -- (2,-4.98) node[anchor= north] {$2/3$};
      \draw (1,-5.02) -- (1,-4.98) node[anchor= north] {$1/3$};
    \draw (3,-5.02) -- (3,-4.98) node[anchor= north] {}; 
 \node at (3,-5.2) {$1$};  
    \node at (0.2,-5.2) {$0$};
  \draw[black,fill=black] (0,-5) circle (.5ex);
     \draw[black,fill=black] (1,-4) circle (.5ex);
     \draw[black,fill=black] (2,-2) circle (.5ex);
     \draw[black,fill=black] (3,-5) circle (.5ex); 
   
     %%% g'
     \draw[->] (5,-5) -- (5,-1.8) node[anchor=west] {$g_{\text{c}}(x)$};
    \draw[->] (5,-5) -- (8.2,-5) node[anchor= west] {$x$};
    \draw (5.07,-4) -- (4.93,-4) node[anchor= east] {$1/3$};
     \draw (5.07,-3) -- (4.93,-3) node[anchor= east] {$2/3$};
      \draw (5.07,-2) -- (4.93,-2) node[anchor= east] {$1$};
    \draw (7,-5.02) -- (7,-4.98) node[anchor= north] {$2/3$};
      \draw (6,-5.02) -- (6,-4.98) node[anchor= north] {$1/3$};
    \draw (8,-5.02) -- (8,-4.98) node[anchor= north] {}; 
 \node at (8,-5.2) {$1$};   
    \node at (5.2,-5.2) {$0$};
  \draw[black,fill=black] (5,-5) circle (.5ex);
     \draw[black,fill=black] (6,-4) circle (.5ex);
     \draw[black,fill=black] (7,-2) circle (.5ex);
     \draw[black,fill=black] (8,-5) circle (.5ex); 
      \draw (5,-5) -- (6,-4) node[anchor= north] {};
    \draw (7,-2) -- (6,-4) node[anchor= north] {};
    \draw (7,-2) -- (8,-5) node[anchor= north] {};
\end{tikzpicture}
}
\caption{The transition function $g$ and the continuous function $ g_{\text{c}}$ in \exmpref{example McNaughton is not enough}.}\label{fig:examplegx}
\end{figure}
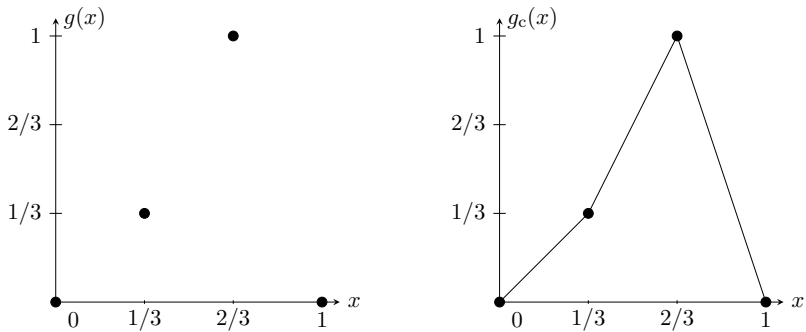
\end{exmp}

Example~\ref{example McNaughton is not enough} manifests missing structure for the MV algebra $\mathcal{I}$ to allow universal conversion from CA transition functions to MV terms. Concretely, inspection of the state set $K$ shows that we would need a McNaughton theorem that allows rational coefficients in Condition 2 of~\thmref{them:McNaughtonTheorem}. It turns out that so-called divisible many-valued (DMV) algebras resolve this issue.

\begin{defn}\label{defn:DivisibleMV01}
Consider the MV algebra $\mathcal{I} = \langle [0,1], \oplus, \lnot,0 \rangle $ in~\defref{defn:MV01}. Define the family of unary operations $\{\delta_i:[0,1]\rightarrow [0,1] \}_{i \in \mathbb{N}}$ according to \begin{equation*}
    \delta_i x = \frac{1}{i}x,\quad x\in[0,1], \quad \text{ for all } i\in \mathbb{N}.
\end{equation*} 
It is easily verified that the structure $\mathcal{I}_{\text{d}}=\langle [0,1], \oplus, \lnot,0,  \{\delta_i \}_{i \in \mathbb{N}}\rangle$ is a DMV algebra \cite{gerla2001divisible}. In particular, it constitutes the algebraic counterpart of Rational \Luka logic \cite{gerla2001divisible}. 
\end{defn}

The following result, often referred to as rational McNaughton theorem, explicitly characterizes the class of term functions under the DMV algebra~$\mathcal{I}_{\dt}$.

\begin{thm}\label{rationalMcNaughtonTheorem}[Rational McNaughton theorem \cite{amato2005neural}]
Consider the DMV algebra $\Ical_{\dt}=\langle [0,1], \oplus, \lnot,0,  \{\delta_i \}_{i \in \mathbb{N}}\rangle$ in~\defref{defn:DivisibleMV01}. Let $n \in \mathbb{N}$. For a function~$f_{\text{c}}: [0,1]^n \rightarrow [0,1]$ to have an associated DMV term $\tau$ such that $\tau^{\Ical_{\dt}} = f_{\text{c}}$ on $[0,1]^n$, it is necessary and sufficient that 
%$f_{\text{c}}$ satisfy the following conditions:
\begin{enumerate}
\item $f_{\text{c}}$ is continuous with respect to the natural topology on $[0,1]^n$,
    \item there exist linear functions $p_1, \ldots \hspace{-0.02cm}, p_\ell$ with rational coefficients, i.e.,
    \begin{equation*}
        p_j(x_1, \ldots \hspace{-0.02cm}, x_{n}) = m_{j1}x_1+\cdots+m_{jn}x_{n}+b_j, \quad j=1,\ldots \hspace{-0.02cm},\ell,
    \end{equation*} where $b_j, m_{j1},\ldots \hspace{-0.02cm}, m_{jn} \in \mathbb{Q}$, for $j=1,\ldots \hspace{-0.02cm}, \ell$, such that for every $x \in [0,1]^n$, there is a $j \in \{1,\ldots \hspace{-0.02cm}, \ell\}$ with $f_{\text{c}}(x) = p_j(x)$.
\end{enumerate}
\end{thm}
The central result we have been working towards in this section now follows readily.

\begin{prop}\label{them:rationalMcnaughton} 
\hspace*{-1mm} Consider a CA with cellular space dimension $d$, neighborhood size $n $, state set $K = \{0, 1/(k-1),\ldots \hspace{-0.02cm}, (k-2)/(k-1), 1\}$ of cardinality $k\in \mathbb{N}, k\geq 2$, and transition function $f: K^{n} \rightarrow K$. There exists a DMV term $\tau(x_1,\ldots \hspace{-0.02cm}, x_n)$ with associated term function under the DMV algebra $\mathcal{I}_{\text{d}}$ satisfying \begin{equation*}
    \tau^{\Ical_{\dt}}(x_1,\ldots \hspace{-0.02cm}, x_n) = f(x_1,\ldots \hspace{-0.02cm}, x_n), \quad \text{ for all } (x_1,\ldots \hspace{-0.02cm}, x_n) \in K^n.
\end{equation*} 
\end{prop}

\begin{proof}
We start by linearly interpolating the CA transition funtion $f: K^n \rightarrow K$ to a continuous function $f_{\text{c}}: [0,1]^n\rightarrow [0,1]$. As all the points \begin{equation*}
    \left(x_1,\ldots \hspace{-0.02cm}, x_n, f(x_1,\ldots \hspace{-0.02cm}, x_n)\right)\in K^{n+1}
\end{equation*} have exclusively rational coordinates, the linear pieces of the interpolated function~$f_{\text{c}}$ have rational coefficients as well. Application of~\thmref{rationalMcNaughtonTheorem} to $f_{\text{c}}$ then yields a DMV term $\tau$ satisfying $\tau^{\Ical_{\dt}} = f_{\text{c}}$ on~$[0,1]^n$. This implies $\tau^{\Ical_{\dt}} = f_{\text{c}} = f$ on $K^n$, as desired.
\end{proof}

\section{Recurrent neural networks realize CA dynamics}\label{sec:NetworkConstruction}
In this section, we show that recurrent neural networks can realize the overall dynamical behavior of CA. This will be accomplished in two steps. First, we demonstrate that ReLU  networks naturally realize operations in DMV algebras, which, in turn, leads to a universal realization theorem for CA transition functions. The ReLU network realizing the CA transition function is then embedded into an RNN that emulates the dynamics of the CA.
%Second, we show that the corresponding CA evolution can be emulated by an RNN.

\subsection{DMV algebras and ReLU neural networks}\label{sec:ReLULearnMV}
We start by formally defining ReLU neural networks \cite{elbrachter2021deep}.
\begin{defn}[ReLU neural network]\label{defn:ReLU}
Let $L \in \mathbb{N}$ and $N_0,$ $N_1, \ldots \hspace{-0.02cm},$ $N_L \in \mathbb{N}$. A ReLU neural network is a map $\Phi: \mathbb{R}^{N_0} \rightarrow \mathbb{R}^{N_L} $ given by 
\[ \Phi= \begin{cases} 
      W_1, & L=1 \\
      W_2 \circ \rho \circ W_1, & L=2 \\
      W_L \circ \rho \circ W_{L-1} \circ \rho \circ \cdots \circ \rho  \circ W_1, & L \geq 3
   \end{cases} \text{,}
\] where, for $\ell \in \{ 1,2,\ldots \hspace{-0.02cm}, L\}$, $W_\ell: \mathbb{R}^{N_{\ell-1}}\rightarrow \mathbb{R}^{N_\ell}, W_\ell(x):= A_\ell x+b_\ell $ are affine transformations with weight matrices $A_\ell = \mathbb{R}^{N_\ell\times N_{\ell-1}}$ and bias vectors $b_\ell \in \mathbb{R}^{N_\ell}$, and the ReLU activation function $\rho: \mathbb{R}\rightarrow \mathbb{R},\rho\hspace{0.02cm}(x):=\max\{0,x\}$ acts component-wise. We denote by $\mathcal{N}_{d, d'}$ the set of ReLU neural networks of input dimension $N_0=d$ and output dimension $N_L = d'$. The number of layers of the network $\Phi$, denoted by $\mathcal{L}(\Phi)$, is defined to equal $L$.
\end{defn}

It follows immediately from \defref{defn:ReLU} that ReLU networks realize continuous piecewise linear functions. We shall frequently make use of basic ReLU network constructions, namely compositions \cite[Lemma II.3]{elbrachter2021deep}, augmentations \cite[Lemma II.4]{elbrachter2021deep}, and parallelizations \cite[Lemma A.7]{elbrachter2021deep}, collected here for completeness. The proofs of the corresponding results~\lemref{lem:ReLUConcatenating}-\ref{lem:parallelizing} provided in \cite{elbrachter2021deep} present explicit network constructions that we will occasionally refer to.

\begin{lem}[Composition of ReLU networks \cite{elbrachter2021deep}]\label{lem:ReLUConcatenating}
Let $d_1, d_2, d_3 \in \mathbb{N}$, $\Phi_1 \in \mathcal{N}_{d_1,d_2}$, and $\Phi_2 \in \mathcal{N}_{d_2,d_3}.$ There exists a network $\Psi \in \mathcal{N}_{d_1,d_3}$ with $\mathcal{L}(\Psi) = \mathcal{L}(\Phi_1)+\mathcal{L}(\Phi_2)$, satisfying \begin{equation*}
    \Psi(x) = (\Phi_2\circ \Phi_1)(x), \quad \text{ for all } x \in \mathbb{R}^{d_1}.
\end{equation*}
\end{lem}
\begin{lem}[Augmentation of ReLU networks \cite{elbrachter2021deep}]\label{lma:augmenting}
 Let $d_1, d_2, L \in \mathbb{N}$, and $\Phi \in \mathcal{N}_{d_1,d_2}$ with $\mathcal{L}(\Phi)<L$. There exists a network $\Psi \in \mathcal{N}_{d_1,d_2}$ with $\mathcal{L}(\Psi)=L$, satisfying $\Psi(x) = \Phi(x)$, for all $x\in \mathbb{R}^{d_1}$.
\end{lem}

\begin{lem}[Parallelization of ReLU networks \cite{elbrachter2021deep}]\label{lem:parallelizing}
Let $n, d, L \in \mathbb{N}$ \linebreak
and, for $i \in \{1,\ldots \hspace{-0.02cm}, n\}$, let $d_i'\in \mathbb{N}$ and $\Phi_i \in \Ncal_{d,d_i'}$ with $\mathcal{L}(\Phi_i) = L$. There exists a network $\Psi \in \Ncal_{d, \sum_{i=1}^nd_i'}$ with $\mathcal{L}(\Psi) = L$, satisfying \begin{equation*}
        \Psi(x) = (\Phi_1(x),\ldots \hspace{-0.02cm}, \Phi_n(x)) \in \mathbb{R}^{\sum_{i=1}^nd_i'}, \quad \text{ for all } x\in \mathbb{R}^{d}.
    \end{equation*}
\end{lem}

We proceed with the  ReLU network constructions realizing operations in the DMV algebra $\mathcal{I}_{\text{d}}=\langle [0,1], \oplus, \lnot,0, \{\delta_i \}_{i \in \mathbb{N}}\rangle$. To this end, we start by noting that the operation $\lnot x = 1-x$ is trivially implemented by a ReLU network with one layer according to
\begin{equation*}
    \lnot x = W_1(x) = 1-x, \quad \text{ for } x \in [0,1].
\end{equation*}
By the same token, as the operations $\delta_i x = \frac{1}{i}x$, for $i\in \mathbb{N}$, are affine mappings, they can be realized by single-layer ReLU networks. The following lemma details the ReLU network constructions realizing the operations $x\oplus y =\min\{1,x+y\}$ and $x\odot y = \max\{0,x+y-1\}$ in $\mathcal{I}_{\text{d}}$.

\begin{lem}\label{lma:MinMax}
There exist ReLU networks $\Phi^{\oplus}\in \mathcal{N}_{2,1}$ and $\Phi^{\odot } \in \mathcal{N}_{2,1}$ satisfying \begin{align*}
\label{eq:networkAND}    \Phi^{\oplus }(x,y) &= \min\{1,x+y\} \\
    \Phi^{\odot }(x,y) &= \max\{0,x+y-1\},
\end{align*} for all $x,y\in [0,1].$
\end{lem}  
\begin{proof}
First, to realize the operation $x\oplus y =\min\{1,x+y\} $, we note that addition can be implemented by a single-layer ReLU network according to
\begin{equation*}
    x+y =  \begin{pmatrix} 1 & 1 \end{pmatrix}  \begin{pmatrix} x \\y \end{pmatrix}.
\end{equation*}
For the ``$\min$'' operation, we observe that
%with a ReLU network. To this end, we observe that
\begin{equation*}
    \min\{1,x\} = 1-\rho(1-x) = (W_2 \circ \rho \circ W_1)(x), \quad x \in [0,1],
\end{equation*} 
where
\begin{equation*}
    W_1(x) = -x+1, \quad \quad  W_2(x) = -x+1.
\end{equation*} 
Now, applying \lemref{lem:ReLUConcatenating}
% \cite[Lemma II.3]{elbrachter2021deep} 
to concatenate the networks  \linebreak
$\Phi_1(x,y) = \begin{pmatrix} 1 & 1 \end{pmatrix}  \begin{pmatrix} x \\y \end{pmatrix} $ and $\Phi_2(x) = (W_2 \circ \rho \circ W_1) (x)$ yields the desired ReLU network realization of $x\oplus y$ according to
\begin{equation*}
    x \oplus y = (W_2^{\oplus} \circ \rho \circ W_1^{\oplus})(x,y), \quad x,y\in [0,1],
\end{equation*} 
where
\begin{equation*}
   W_1^{\oplus}(x,y)= \begin{pmatrix}
       -1 & -1
   \end{pmatrix} \begin{pmatrix}
       x \\y
   \end{pmatrix}+1, \quad
   W_2^{\oplus}(x) = -x+1.
\end{equation*}

To realize the operation $x\odot y = \max\{0,x+y-1\}$, we directly note that \begin{equation*}
    \max\{0,x+y-1\} =  \rho \left( \begin{pmatrix}
        1 & 1
    \end{pmatrix}   \begin{pmatrix}
      x \\ y
    \end{pmatrix} -1 \right) = (W_2^{\odot} \circ \rho \circ W_1^{\odot})(x,y),
\end{equation*} for $x,y\in[0,1]$, where \begin{equation*}
    W_1^{\odot} (x,y) = \begin{pmatrix}
        1 & 1
    \end{pmatrix}   \begin{pmatrix}
      x \\ y
    \end{pmatrix} -1,\quad \quad  W_2^{\odot}(x) = x.
\end{equation*}
\end{proof}

Equipped with the ReLU network realizations of the logical operations underlying $\mathcal{I}_{\text{d}}$,
%and hence also $\mathcal{I} = \langle [0,1], \oplus,  \lnot,0\rangle$, 
we are now ready to state the following universal representation result.

\begin{prop}\label{prop:ReLU realize divisible}
Consider the DMV algebra $\mathcal{I}_{\text{d}}$ in~\defref{defn:DivisibleMV01}. Let $n \in \mathbb{N}$. For each DMV term $\tau(x_1,\ldots \hspace{-0.02cm}, x_n)$ and its associated term function $\tau^{\mathcal{I}_{\text{d}}}:[0,1]^n \rightarrow [0,1]$, there exists a ReLU network $\Phi \in \mathcal{N}_{n,1}$ such that \begin{equation*}
    \Phi(x_1,\ldots \hspace{-0.02cm}, x_n) = \tau^{\Ical_{\dt}}(x_1,\ldots \hspace{-0.02cm}, x_n), \text{ for all } (x_1,\ldots \hspace{-0.02cm}, x_n)\in [0,1]^n.
\end{equation*}
\end{prop}
\begin{proof}
The proof follows by realizing the logical operations appearing in the term function $\tau^{\mathcal{I}_{\text{d}}}$ through corresponding concatenations, according to \lemref{lem:ReLUConcatenating}, of ReLU networks implementing the operations $\oplus$ and $\odot$ as by \lemref{lma:MinMax} and noting that $\lnot x= 1-x, \delta_ix = \frac{1}{i}x$ are trivally ReLU networks with one layer.
\end{proof}

We note that, in general, the ReLU network $\Phi$ in \propref{prop:ReLU realize divisible} will be a properly deep network as it is obtained by concatenating the networks realizing the basic logical operations $\lnot$, $\oplus, \odot$, and $\{\delta_i \}_{i \in \mathbb{N}}$. 

Now the ground has been prepared for the central result in this section, namely a universal ReLU network realization theorem for CA transition functions. Specifically, this will be effected by combining the connection between CA and DMV algebras established in~\secref{sec:MVLogic} with the ReLU network realizations of the logical operations in DMV algebras presented above. 

\begin{thm}\label{thm:CA-ReLU}
Consider a CA with cellular space dimension $d \in \mathbb{N}$, neighborhood size $n \in \mathbb{N}$, state set $K = \{0,\frac{1}{k-1},\ldots \hspace{-0.02cm}, \frac{k-2}{k-1}, 1\}$ of cardinality $k \in \mathbb{N}, k\geq 2$, and transition function $f: K^{n} \rightarrow K$. There exists a ReLU network $\Phi \in \mathcal{N}_{n,1}$ satisfying 
\begin{equation*}
    \Phi(x_1,\ldots \hspace{-0.02cm}, x_n) = f(x_1,\ldots \hspace{-0.02cm}, x_n),\quad \text{ for all } (x_1,\ldots \hspace{-0.02cm},x_n)\in K^n.
\end{equation*}
\end{thm}

\begin{proof}
%Consider the transition function $f$. 
By \propref{them:rationalMcnaughton}, there exists a  DMV term $\tau(x_1,\ldots \hspace{-0.02cm}, x_n)$ whose corresponding term function 
under the DMV algebra $\Ical_{\dt}$ satisfies 
\begin{equation}\label{eq:15}
    \tau^{\Ical_{\dt}}(x_1,\ldots \hspace{-0.02cm},x_n) = f(x_1,\ldots \hspace{-0.02cm},x_n),\quad \text{ for all } (x_1,\ldots \hspace{-0.02cm}, x_n) \in K^n.
\end{equation} 
Application of \propref{prop:ReLU realize divisible} then yields a ReLU network $\Phi \in \mathcal{N}_{n,1}$ such that \begin{equation}\label{eq:16}
    \Phi(x_1,\ldots \hspace{-0.02cm},x_n) = \tau^{\Ical_{\dt}}(x_1,\ldots \hspace{-0.02cm},x_n), \quad \text{ for all } (x_1,\ldots \hspace{-0.02cm},x_n)\in [0,1]^n.
\end{equation} The proof is finalized by combining~\eqref{eq:15} and~\eqref{eq:16} to get
\begin{equation*}
    \Phi(x_1,\ldots \hspace{-0.02cm},x_n) = \tau^{\Ical_{\dt}}(x_1,\ldots \hspace{-0.02cm},x_n) = f(x_1,\ldots \hspace{-0.02cm},x_n),
\end{equation*}
for all $(x_1,\ldots \hspace{-0.02cm},x_n) \in K^n$.
\end{proof}

\iffalse 
\begin{exmp}\label{exmp:2.2}
Consider a CA with cellular space $\mathbb{Z}$, the neighborhood set $\mathcal{E}=\{-1,0,1\}\subset \mathbb{Z}$, and state set $K$ in~\eqref{eq:KStateSet} with $k\geq 3$. The transition rule $f$ is given by 
\begin{equation*}\label{eq:kcolorRule}
    f(x_{-1},x_0,x_1) =\begin{cases}
(x_0 + \frac{1}{k-1}) \mathbbm{1}_{ \{ x_{-1} = x_0 +\frac{1}{k-1} \text{ or } x_1 = x_0+\frac{1}{k-1}\}} &\quad \\
\hspace{0.7cm}  +\; x_0\mathbbm{1}_{\{ x_{-1} \neq x_0+\frac{1}{k-1}\} }\mathbbm{1}_{ \{ x_{1} \neq x_0+\frac{1}{k-1}\} }, & \text{ if } x_0 \in K \text{\textbackslash} \{1\}, \\ 
x_0\mathbbm{1}_{\{ x_{-1} \neq 0\} }\mathbbm{1}_{ \{ x_{1} \neq 0\} }, & \text{ if } x_0=1.
    \end{cases}
\end{equation*} 
which has a corresponding MV term under the MV algebra $\mathcal{A}_k$ in~\exmpref{exmp:MVAlgebraAk}. The resulting CA is the $k$-colored cyclic CA created by Bramson and Griffeath \cite{bramson1989flux} to model particle systems. 
\end{exmp}
\fi

\subsection{Realizing the dynamical behavior of CA}\label{sec:3.2}
We now turn our attention to the dynamical systems aspects of CA. Concretely, we show how CA evolution can be realized through RNNs. The basic idea underlying the construction we present is to suitably impose a recurrent structure on top of the ReLU network realizing the transition function of the CA under consideration. Concretely, we build on the RNN construction techniques developed in \cite{hutter2022metric} and, for simplicity of exposition, consider the case of one-dimensional CA. An extension to the multi-dimensional case is quite readily obtainable using tools from multi-dimensional signal processing \cite{vaidyanathan2006multirate} and multi-dimensional RNNs \cite{graves2008offline}, \cite{leifert2014cells}. For the sake of simplicity of exposition, however, we do not provide these extensions here.

An RNN is a discrete dynamical system mapping an input sequence to an output sequence, both possibly of infinite length, through---at each discrete time step---application of a feedforward neural network that updates a hidden-state vector and computes the next output signal sample {\cite{elman1990finding}. The formal definition of an RNN is as follows. 
%%%Y: please replace the ref. to my paper with Clemens and Recep by a generic reference to RNNs, as discussed.

\begin{defn}[Recurrent neural network]\label{RNN}
For hidden-state vector\linebreak
dimension $m \in \mathbb{N}$, let $\Phi \in \mathcal{N}_{m+1,m+1}$ be a ReLU neural network. The recurrent neural network associated with $\Phi$ is the operator $\mathcal{R}_{\Phi}$ mapping input sequences $(x[z])_{z \in \mathbb{N}_0}$ in $\mathbb{R}$ to output sequences $(y[z])_{z \in \mathbb{N}_0}$ in $\mathbb{R}$ according to \begin{equation}\label{eq:RNNdefinition}
   \begin{pmatrix} y[z] \\ h[z] \end{pmatrix} = \Phi  \left( \begin{pmatrix}
   x[z] \\ h[z-1]
   \end{pmatrix}  \right), \quad z \in \mathbb{N}_0,
\end{equation} where $h[z] \in \mathbb{R}^{m}$ is the hidden-state vector with initial value $h[-1] = 0_{m}$.
\end{defn}

The key to emulating the evolution of a CA with an RNN is  the construction of an appropriate hidden-state vector in combination with a suitable ReLU network that encodes the CA transition function. Before stating the corresponding result, by way of preparation, we introduce a decomposition of
$\Phi$ in~\eqref{eq:RNNdefinition} according to
%. Specifically, we shall split up $\Phi$ according to 
\begin{equation}\label{eq:C}
    \Phi = \begin{pmatrix}
        \Phi^f \\
        \Phi^h
    \end{pmatrix},
\end{equation} 
where $\Phi^f \in \Ncal_{m+1,1}$ is responsible for the computation of the output sample according to 
\begin{equation}\label{eq:ypart}
    y[z] = \Phi^f  \left( \begin{pmatrix}
   x[z] \\ h[z-1]
   \end{pmatrix}  \right), \quad z \in \mathbb{N}_0,
\end{equation}
and $\Phi^h \in  \Ncal_{m+1,m}$ effects the evolution of the hidden-state vector such that
\begin{equation}\label{eq:hpart}
     h[z] = \Phi^h  \left( \begin{pmatrix}
   x[z] \\ h[z-1]
   \end{pmatrix}  \right), \quad z \in \mathbb{N}_0.
\end{equation}
The following theorem states the announced universal realization theorem for one-dimensional CA by RNNs. 

\begin{thm}\label{thm:RNNRealize}
Consider a CA with cellular space dimension $d =1$, neighborhood set $\mathcal{E}=\{0,-1,\ldots \hspace{-0.02cm}, -n+1\}$ of size $n \in \mathbb{N}, n\geq 2$, state set $K = \{0,\frac{1}{k-1},\ldots \hspace{-0.02cm}, \frac{k-2}{k-1}, 1\}$ of cardinality $k \in \mathbb{N}, k\geq 2$, and transition function $f: K^{n} \rightarrow K$ with associated CA map $F$. There exists an RNN that maps every configuration $(c[z])_{z\in \mathbb{Z}}$ over the cellular space $\mathbb{Z}$ to the next configuration $\left(F(c)[z]\right)_{z\in \mathbb{Z}}$ according to 
\begin{equation}\label{eq:GlobalMap2}
    F(c)[z] = f(c[z], c[z-1], \ldots \hspace{-0.02cm}, c[z-n+1]), \quad  \forall z\in \mathbb{Z}.
\end{equation}
\end{thm}

\begin{rem}
Note that here we do not work with the general neighborhood set $\mathcal{E}=\{0, z_1,\ldots \hspace{-0.02cm}, z_{n-1}\} \subset \mathbb{Z}$, but rather consider the concrete neighborhood set $\mathcal{E} = \{0, -1, \ldots \hspace{-0.02cm}, -n+1\}$. This does not result in a loss of generality as the following argument shows. Take a general neighborhood set $\{0,z_1,\ldots \hspace{-0.02cm},z_{n-1}\} \subset \mathbb{Z}$ and expand it into a set of the form 
\begin{equation}\label{eq:neighborhood}
   \{\ell_1, \ell_1-1,\ldots \hspace{-0.02cm}, \ell_2+1,\ell_2\}, 
\end{equation}
where $\ell_1:=\max\{0,z_1,\ldots \hspace{-0.02cm},z_{n-1}\}$ and $\ell_2 := \min\{0,z_1,\ldots \hspace{-0.02cm},z_{n-1}\}$ by, when needed, adding void neighbors that do not affect the CA transition function. For example, consider the CA with neighborhood set $\mathcal{E}_*=\{1,0,-2\}$, transition function~$f_*$, and corresponding global map 
\begin{equation*}\label{eq:ExampleF}
    F_*(c_*)[z] = f_*(c_*[z+1], c_*[z], c_*[z-2]), \quad  \text{ for } (c_*[z])_{z\in \mathbb{Z}}.
\end{equation*} 
Then, expand $\mathcal{E}_*$ to $\mathcal{E}=\{1,0,-1,-2\}$ and make $f_*$ formally depend on the void neighbor $z=-1$ according to 
\begin{multline*}
     f(c_*[z+1], c_*[z],  c_*[z-1], c_*[z-2]) =\\f_*(c_*[z+1], c_*[z], c_*[z-2]),  \text{ for } (c_*[z])_{z\in \mathbb{Z}}.
\end{multline*} 
These modifications lead to the global map 
\begin{equation}\label{eq:examplec}
    F_*(c_*)[z] =  f( c_*[z+1], c_*[z], c_*[z-1], c_*[z-2]), \quad \text{ for } (c_*[z])_{z\in \mathbb{Z}}.
\end{equation}
The general form~\eqref{eq:GlobalMap2} is finally obtained upon making the substitution
\begin{equation*}
     c[z]=c_*[z+1], \text{ for all }z\in \mathbb{Z},
\end{equation*} to yield
\begin{equation*}
    F(c)[z] =  f(c[z], c[z-1], c[z-2], c[z-3]), \quad \text{ for } (c[z])_{z\in \mathbb{Z}}.
\end{equation*} 
\end{rem}

\begin{proof}
%%%Y: do you agree with the first two sentences here?
For ease of exposition, we provide the proof for one-sided infinite sequences $(c[z])_{z\in \mathbb{N}_0}$ only. The general case follows mutatis mutandis, but is based on exactly the same ideas. The proof is constructive and will be effected by explicitly specifying the ReLU networks $\Phi^f$ and $\Phi^h$ in~\eqref{eq:C} underlying the desired RNN. By~\eqref{eq:GlobalMap2}, the output sample $F(c)[z]$, which we identify with $y[z]$, is a function of the states of the cells $\{z, z-1, \ldots \hspace{-0.02cm}, z-n+1\}$. Therefore, equating the current input sample $x[z]$ in~\eqref{eq:ypart} with $c[z]$, we choose the hidden-state vector $h[z-1]$ such that it stores the states of the neighbors $\{z-1,\ldots \hspace{-0.02cm}, z-n+1\}$ of the current input cell, i.e.,
\begin{equation}\label{eq:hiddenvec}
    h[z-1] = \begin{pmatrix}
        c[z-1] \\ c[z-2]\\ \vdots \\ c[z-n+1]
    \end{pmatrix}.
\end{equation} 
This leads to 
\[
h[z] = \begin{pmatrix}
        c[z] \\ c[z-1]\\ \vdots \\ c[z-n+2]
    \end{pmatrix}
\]
and informs the choice of $\Phi^h$, which must satisfy
\begin{equation*}
    \Phi^h \left(\begin{pmatrix}
    c[z] \\ h[z-1]
    \end{pmatrix} \right) = \Phi^h \left( \begin{pmatrix}
    c[z] \\ c[z-1] \\\vdots \\c[z-n+2] \\ c[z-n+1]
    \end{pmatrix}\right) =  \begin{pmatrix}
    c[z] \\ c[z-1] \\\vdots \\ c[z-n+2] 
    \end{pmatrix} = h[z].
\end{equation*}
The evolution of the hidden-state vector hence proceeds by dropping the oldest value $c[z-n+1]$ and inserting the new value $c[z]$ at the top. Following the methodology developed in \cite{hutter2022metric}, $\Phi^h$ can be realized by a two-layer ReLU network according to 
\begin{equation*}
\Phi^h: \mathbb{R}^{n} \rightarrow \mathbb{R}^{n-1}, \quad
    \Phi^h(x)= (W^h_2 \circ \rho \circ W^h_1) (x),
\end{equation*} with affine maps $W^h_1: \mathbb{R}^{n} \rightarrow \mathbb{R}^{2n}$, $W^h_2: \mathbb{R}^{2n} \rightarrow \mathbb{R}^{n-1} $ given by 
\begin{equation}\label{eq:HiddenNetwork}
\begin{aligned}
     W^h_1(x) &= \begin{pmatrix}
    \mathbb{I}_{n} \\ - \mathbb{I}_{n}
    \end{pmatrix} x, \quad \text{for } x \in \mathbb{R}^{n},\\
     W^h_2(x) &= \begin{pmatrix}
    \mathbb{I}_{n-1} \quad 0_{n-1} \quad -\mathbb{I}_{n-1} \quad 0_{n-1}
    \end{pmatrix}x,\quad \text{for } x \in  \mathbb{R}^{2n}.
\end{aligned}
\end{equation} 
We are left with the design of the network $\Phi^f$ which has to satisfy 
\begin{equation*}
    F(c)[z]   = \Phi^f  \left( \begin{pmatrix}
   c[z] \\ c[z-1] \\ \vdots \\ c[z-n+1]
   \end{pmatrix}  \right).
\end{equation*} 
It follows by inspection of~\eqref{eq:GlobalMap2} that this amounts to realizing the CA transition function $f$ through the ReLU network $\Phi^f$. Based on this insight, we can apply \thmref{thm:CA-ReLU} to conclude the existence of $\Phi^f$. In fact, the proof of~\thmref{thm:CA-ReLU} spells out how $\Phi^f$ can be obtained explicitly by composing the logical operations in the DMV term associated with~$f$. The proof is now completed by noting that the overall network $\Phi$ is obtained by applying \lemref{lma:augmenting} and \lemref{lem:parallelizing} to combine $\Phi^f$ and $\Phi^h$ according to~\eqref{eq:C}.
\end{proof}

%%%%%%%%%%%%%%%%%%%%%%%%%%%%%%%%%%%%%%%%%%%%%%%%%%%%%%%%%%%%%%%%%%%%%%%%%%%%%%%%%
%%%%%%%%%%%%%%%%%%%%%%%%%%%%%%%%%%%%%%%%%%%%%%%%%%%%%%%%%%%%%%%%%%%%%%%%%%%%%%%%%
%%%%%%%%%%%%%%%%%%%%%%%%%%%%% section 4 %%%%%%%%%%%%%%%%%%%%%%%%%%%%%%%%%%%%%%%%%
\section{Identification of CA logic from trained neural networks}\label{sec:Extract}
With suitably chosen transition functions and initial configurations, CA can simulate a plethora of dynamical behavior characteristics \cite{wolfram2002new}. The inverse problem of deducing CA transition functions from observations of their evolution is, however, extremely difficult \cite{adamatzky2010game}, \cite{adamatzky2018identification}.
Formally, this is known as the CA identification problem \cite[Section~1.4]{adamatzky2018identification}: Given a finite sequence of consecutive configurations $\{c^1,c^2,\ldots \hspace{-0.02cm}, c^T\}$ collected during the evolution, construct a CA $(\mathbb{Z}^d, K, \mathcal{E}, f)$ whose associated CA map $F$ satisfies
\[
F(c^t) = c^{t+1}, \quad \text{ for } t=1, \ldots \hspace{-0.02cm}, T-1.
\]
%%%Y: here we should comment on the impossibility results of going back to initial configurations, by giving at least one example,
%could you kindly suggest a text
Note that the initial configuration need not be recoverable from observed evolution traces \cite[Section 3]{wolfram1983statistical}, \cite[Section~4]{kari2005theory}. 
%A detailed exposition of related concepts, namely injectivity, surjectivity, and bijectivity of CA can be found in 
%depending on the reversibility of the underlying CA \cite[Sec. 3]{wolfram1983statistical}\cite[Sec.~4]{kari2005theory}, the initial configuration is not %always recoverable from the observed evolution traces. 
%To illustrate this point, consider a simple example, namely the elementary CA of rule $0$, which maps all possible configurations to the all-$0$ %configuration after one step of evolution. 

Whilst the cellular space $\mathbb{Z}^d$ and the state set $K$ can be read off directly from 
evolution traces, the neighborhood set $\mathcal{E}$ can either be selected manually \cite{RICHARDS1990189}, \cite{adamatzky1997automatic} or determined 
through the application of specific criteria such as, e.g., mutual information \cite{zhao2006neighborhood}. The most challenging aspect of the CA identification problem resides in determining the transition function~$f$. 
\iffalse 
Sakama et al. \cite{maeda2007identifying} construct decision trees to classify pairs of combinations of neighborhood states and new states collected from the evolution traces, which is then used to derive the transition function by genetic programming. In the case of binary state set $K = \{0,1\}$, Zhao et al. \cite{article} exploit the fact that the transition function can be represented by a polynomial, which takes states of neighborhood cells as input variables, to reduce the problem from searching through all possible transition functions to computing the polynomial coefficients. \textcolor{blue}{Either write:} 
\fi 
%We refer to \cite{Adamatzky2009} for 
A detailed discussion of known CA identification methods can be found in~\cite{Adamatzky2009}.
%%%Y: should we cite any other refs. beyond Adamatzky?
%as well as automatic discovery of CA transition functions for certain tasks. \textcolor{blue}{or write:} The more general problem of automatic discovery of %CA transition functions to perform certain tasks has been extensively studied using evolutionary computing \cite{das1995evolving} and genetic %programming \cite{andre1996discovery, el2007genetic, mitchell1994evolving} by selecting from all possible transition functions those that most accurately %compute the specified task. 

We next propose a novel approach to CA identification, namely reading out the DMV formula underlying the CA under consideration from an RNN 
trained on its evolution traces. The focus will be on elucidating the fundamental principles of this idea.
Accordingly, we will not be concerned with the performance of specific RNN training algorithms, but will rather
assume that the feedforward network part $\Phi^{f}$ inside the RNN has been trained to achieve what is called in machine learning parlance ``interpolation'', i.e.,
\begin{equation}\label{eq:interpolation}
    \Phi^{f}(x_1,\ldots \hspace{-0.02cm},x_n) = f(x_1,\ldots \hspace{-0.02cm},x_n),
\end{equation}
for all $(x_1,\ldots \hspace{-0.02cm},x_n) \in K^n$.
This, of course, requires that the RNN being trained 
have ``seen'' all possible combinations of neighborhood states and thereby the transition function $f$ on its entire domain $K^{n}$, a condition
met when the training configuration sequences are sufficiently long and their initial configurations exhibit sufficient richness \cite[Thesis 3.1, Thesis 3.2]{adamatzky2018identification}.
%This condition can be guaranteed when e.g., the collected configurations, which are of finite length in practice, are of enough length
% \cite[Thesis 3.1, Thesis 3.2]{adamatzky2018identification}. The question of how well the CA transition function can be recovered from a neural network %trained on evolution traces that otherwise do not meet this assumption, is left to future work. 
We will make this a standing assumption. In addition, the cellular space, the state set, and the neighborhood set are taken to be known a priori.

%\textcolor{blue}{written by Yani starts}
\subsection{Interpolation}

The procedure for extracting DMV terms underlying CA evolution data presented in Section~\ref{sec:reading-out} below works off $\Phi^f$ as a (continuous) function mapping $[0,1]^{n}$ to $[0,1]$.
It turns out, however, that condition~\eqref{eq:interpolation} does not uniquely determine $\Phi^f$ on $[0,1]^n$ as there are---in general infinitely 
many---different ways of interpolating $f(x_1,\ldots \hspace{-0.02cm},x_n)$ to a continuous piecewise linear function; recall that ReLU networks always realize continuous piecewise linear functions.
%even when restricting
%to continuous piecewise linear $\Phi^{f}$ accounting for the fact that ReLU networks exclusively realize continuous piecewise linear functions. 
Since \thmref{rationalMcNaughtonTheorem} states that the truth functions associated with DMV terms under $\mathcal{I}_{\text{d}}$ are continuous piecewise linear functions with rational coefficients, we can constrain the weights of $\Phi^{f}$ to be rational. In fact, as the next two lemmata show, when interpolating CA transition functions, we can tighten this constraint even further, 
namely to integer weights and rational biases. We first establish that CA transition functions can be interpolated to continuous piecewise linear functions with integer weights and rational biases\footnote{To be consistent with terminology used in the context of neural networks, for a linear function of the form $p(x_1, \ldots \hspace{-0.02cm}, x_{n}) = m_{1}x_1+\cdots+m_{n}x_{n}+b$, the coefficients $m_i$ will often be referred to as weights and $b$ as bias.}.
Then, it is shown that such functions can be realized by ReLU networks with integer weights and rational biases.

\iffalse 
as stated in \lemref{lma:refinefc}. Based on Condition  2 of \thmref{rationalMcNaughtonTheorem}, we can constrain that each linear piece of the interpolated functions has rational coefficients. In fact, when the underlying discrete function~$f$ is a CA transition function, we can further tighten the constraint on coefficients, as stated in \lemref{lma:refinefc}.
In fact, generally, there exist infinitely many different $\Phi^f(x),\,x\in[0,1]^{n}$ satisfying~\eqref{eq:interpolation}.
Before presenting our procedure for extracting DMV terms from trained networks, we discuss the interpolation condition \eqref{eq:interpolation} and how it affects the extraction process.
With the goal of extracting DMV terms in mind, we only consider continuous piecewise linear interpolations of the discrete function $f$, because on the one hand, ReLU networks realize piecewise linear functions; on the other hand, the rational McNaughton theorem in \thmref{rationalMcNaughtonTheorem} states that the truth functions associated with DMV terms under $\mathcal{I}_{\text{d}}$ are continuous piecewise linear functions. Based on Condition  2 of \thmref{rationalMcNaughtonTheorem}, we can constrain that each linear piece of the interpolated functions has rational coefficients. In fact, when the underlying discrete function~$f$ is a CA transition function, we can further tighten the constraint on coefficients, as stated in \lemref{lma:refinefc}.
\fi

\begin{lem}\label{lma:refinefc}
Consider a CA with cellular space dimension $d \in \mathbb{N}$, neighborhood size $n \in \mathbb{N}$, state set $K = \{0,\frac{1}{k-1},\ldots , \frac{k-2}{k-1}, 1\}$ of cardinality $k \in \mathbb{N}, k\geq 2$, and transition function $f: K^{n} \rightarrow K$. There exists a function $f_{\text{c}}:[0,1]^n \rightarrow [0,1]$ with the following properties:
\begin{enumerate}
\item $f_{\text{c}}(x_1, \ldots ,x_n) = f(x_1,\ldots \hspace{-0.02cm}, x_n)$, \quad \text{for} $(x_1, \ldots \hspace{-0.02cm},x_n)\in K^n$,
    \item $f_{\text{c}}$ is continuous with respect to the natural topology on $[0,1]^n$,
    \item there exist $\ell \in \mathbb{N}$ and linear functions %$p_1,\ldots ,p_\ell$,
\begin{equation}%\label{eq:coeff}
        p_j(x_1,\ldots \hspace{-0.02cm},x_n) = m_{j1}x_1+\cdots+m_{jn}x_n+\frac{b_j}{k-1}, \quad j=1,\ldots \hspace{-0.02cm},\ell,
\end{equation}
with $b_j,m_{j1},\ldots \hspace{-0.02cm},m_{jn} \in \mathbb{Z}$, for $j=1,\ldots \hspace{-0.02cm}, \ell$ such that, for every $x\in[0,1]^n$ there is an index $j\in \{1,\ldots \hspace{-0.02cm}, \ell\}$ with $f_{\text{c}}(x) = p_j(x)$.
\end{enumerate}
\end{lem} 
\begin{proof}
We explicitly construct a function $f_{\text{c}}$ satisfying Properties 1-3. For $n=1$, one simply performs linear interpolation between
each pair of points $(\frac{i}{k-1}, f(\frac{i}{k-1}))$ and $(\frac{i+1}{k-1}, f(\frac{i+1}{k-1}))$, $i=0,\ldots \hspace{-0.02cm},k-2$, to obtain
an $f_{\text{c}}$ with the desired properties. While it is immediate that the resulting function $f_{\text{c}}$ satisfies Properties $1$ and $2$ and has the structure as demanded by Property $3$, 
the condition $b_j,m_{j1},\ldots \hspace{-0.02cm},m_{jn} \in \mathbb{Z}$ will be verified summarily, for all $n$, at the end of the proof.
%one obtains a continuous piecewise linear interpolation of the discrete function $f:K\rightarrow K$ by drawing a linear polynomial between 

Before proceeding to the cases $n \geq 2$, we note that throughout the proof an $n$-simplex, $n\in \mathbb{N}$, with vertices $\{v_0,v_1,\ldots \hspace{-0.02cm},v_n\} \subset \mathbb{R}^n$ will be denoted by $\Delta^n = [v_0, v_1, \ldots \hspace{-0.02cm}, v_n]$.
We turn to the case $n=2$ and refer to \figref{fig:K2} for an illustration of the domain $K^2$ of $f$ for $k=3$. First, we divide $[0,1]^2$ into the squares $[\frac{i_1}{k-1}, \frac{i_1+1}{k-1}] \times [\frac{i_2}{k-1},\frac{i_2+1}{k-1}]$, $i_1,i_2 \in \{0,\ldots \hspace{-0.02cm},k-2\}$, followed by a subdivision of each of the resulting squares into $2$-simplices. There are precisely two possibilities for each of these subdivisions. To illustrate this, we consider the square $[0,\frac{1}{k-1}]\times [0,\frac{1}{k-1}]$ and note that
it can be split into two $2$-simplices according to either
\begin{align*}
    \Delta^2_1 &= \left[(0,0), \left(0,\frac{1}{k-1}\right), \left(\frac{1}{k-1}, \frac{1}{k-1}\right)\right] \\
    \Delta^2_2 &= \left[(0,0), \left(\frac{1}{k-1},0\right), \left(\frac{1}{k-1}, \frac{1}{k-1}\right)\right]
\end{align*}
or 
\begin{align*}
    \Delta^2_3 &= \left[(0,0), \left(0,\frac{1}{k-1}\right), \left(\frac{1}{k-1}, 0\right)\right] \\
    \Delta^2_4 &= \left[\left(\frac{1}{k-1},\frac{1}{k-1}\right), \left(0,\frac{1}{k-1}\right), \left(\frac{1}{k-1}, 0\right)\right],
\end{align*}
as depicted in \figref{fig:subdide2}.
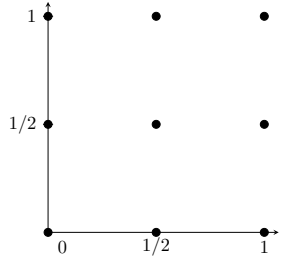
\begin{figure}[t]
\centering
\resizebox{4cm}{!}{

\begin{tikzpicture}[x=1.5cm, y=1.5cm, >=stealth, scale=1, transform shape]

     %%% f'
  
%%% g
     \draw[->] (0,-5) -- (0,-1.8) node[anchor=west] {};
    \draw[->] (0,-5) -- (3.2,-5) node[anchor= west] {};
   \draw (0.07,-3.5) -- (-0.07,-3.5) node[anchor= east] {$1/2$};
      \draw (0.07,-2) -- (-0.07,-2) node[anchor= east] {$1$};

      \draw (1.5,-5.02) -- (1.5,-4.98) node[anchor= north] {$1/2$};
    \draw (3,-5.02) -- (3,-4.98) node[anchor= north] {}; 
 \node at (3,-5.2) {$1$};  
    \node at (0.2,-5.2) {$0$};
  \draw[black,fill=black] (0,-5) circle (.5ex);
     \draw[black,fill=black] (1.5,-5) circle (.5ex);
     \draw[black,fill=black] (3,-5) circle (.5ex); 
       \draw[black,fill=black] (0,-3.5) circle (.5ex);
     \draw[black,fill=black] (1.5,-3.5) circle (.5ex);
     \draw[black,fill=black] (3,-3.5) circle (.5ex); 
  \draw[black,fill=black] (0,-2) circle (.5ex);
     \draw[black,fill=black] (1.5,-2) circle (.5ex);
     \draw[black,fill=black] (3,-2) circle (.5ex);    
\end{tikzpicture}
}
\caption{The set $K^2$ for $k=3$.}\label{fig:K2}
\end{figure}

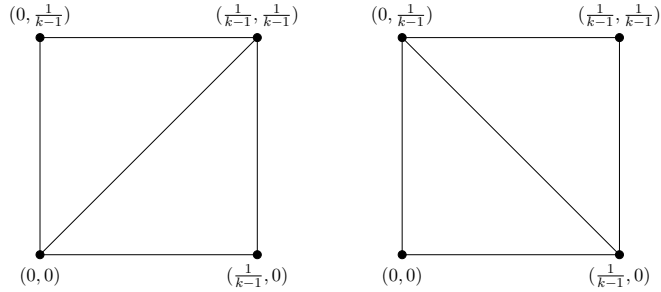
\begin{figure}[t]
\centering
\resizebox{9cm}{!}{

\begin{tikzpicture}[x=1.5cm, y=1.5cm, >=stealth, scale=1, transform shape]

%%% left
     \draw (0,-5) -- (0,-2) node[anchor=west] {};
      \draw (0,-2) -- (3,-2) node[anchor= west] {};
    \draw (0,-5) -- (3,-5) node[anchor= west] {};
       \draw (3,-5) -- (3,-2) node[anchor= west] {};
        \draw (0,-5) -- (3,-2) node[anchor= west] {};
  
 \node at (3,-5.3) {$(\frac{1}{k-1},0)$};  
    \node at (0.0,-5.3) {$(0,0)$};
    \node at (3.0,-1.7) {$(\frac{1}{k-1},\frac{1}{k-1})$};
   \node at (0.0,-1.7) {$(0, \frac{1}{k-1})$}; 
  \draw[black,fill=black] (0,-5) circle (.5ex);
%     \draw[black,fill=black] (1.5,-5) circle (.5ex);
     \draw[black,fill=black] (3,-5) circle (.5ex); 
  \draw[black,fill=black] (0,-2) circle (.5ex);
     \draw[black,fill=black] (3,-2) circle (.5ex);    

%%%right
      \draw (5,-5) -- (5,-2) node[anchor=west] {};
\draw (5,-2) -- (8,-2) node[anchor= west] {};     
 \draw (5,-5) -- (8,-5) node[anchor= west] {};
       \draw (8,-5) -- (8,-2) node[anchor= west] {};
        \draw (5,-2) -- (8,-5) node[anchor= west] {};

        \node at (8,-5.3) {$(\frac{1}{k-1},0)$};  
    \node at (5.0,-5.3) {$(0,0)$};
    \node at (8.0,-1.7) {$(\frac{1}{k-1},\frac{1}{k-1})$};
   \node at (5.0,-1.7) {$(0, \frac{1}{k-1})$}; 
  \draw[black,fill=black] (5,-5) circle (.5ex);

   \draw[black,fill=black] (8,-5) circle (.5ex); 
  \draw[black,fill=black] (5,-2) circle (.5ex);
     \draw[black,fill=black] (8,-2) circle (.5ex);    
     
\end{tikzpicture}
}
\caption{Two different subdivisions.}\label{fig:subdide2}
\end{figure}
%Take the square $[0,\frac{1}{k-1}]\times [0,\frac{1}{k-1}]$ for example, we can 
%as shown in \figref{fig:subdide2}. %We subdivide all mini squares $[\frac{i_1}{k-1}, \frac{i_1+1}{k-1}] \times [\frac{i_2}{k-1},\frac{i_2+1}{k-1}]$, for %$i_1,i_2 \in \{0,\ldots \hspace{-0.02cm},k-2\}$, into $2$-simplices in the same fashion, resulting in a subdivision of the unit square $[0,1]^2$. 
Fixing the subdivisions of all squares, on each $2$-simplex $\Delta^2 =[v_0,v_1,v_2]$, we define a linear function $p_{\Delta^2}$ that interpolates $f(x)$, i.e., $p_{\Delta^2}(x) = f(x)$, for $x\in \{v_0,v_1,v_2\}$. As $\{v_0,v_1,v_2\}$ is affinely independent, $p_{\Delta^2}$ is uniquely determined \cite[Chapter~13]{smith2012linear}. We then stitch the resulting linear pieces together as follows.
%together to obtain a continuous piecewise linear function $f_{\text{c}}$. 
For each point $x\in [0,1]^2$, if it falls into the interior of some $2$-simplex $\Delta^2$, we set $f_{\text{c}}(x) =p_{\Delta^2}(x)$; if $x$ resides in the intersection of two or more $2$-simplices, then $f_{\text{c}}(x)$ is taken to have the (shared) value of the interpolating functions associated with the intersecting simplices evaluated at~$x$. The resulting function $f_{\text{c}}$ satisfies Properties $1$ and $2$ and exhibits the structure as demanded by Property $3$.

For $n\geq 3$, we first divide the unit cube $[0,1]^n$ into the $n$-dimensional cubes $[\frac{i_1}{k-1}, \frac{i_1+1}{k-1}] \times \cdots \times [\frac{i_n}{k-1}, \frac{i_n+1}{k-1}], i_1,\ldots \hspace{-0.02cm},i_n \in \{0,\ldots \hspace{-0.02cm},k-2\}$.
%into simplices with vertices $K^n$, and then constructing linear polynomials on each simplex. Specifically, first the set $K^n$ splits the unit cube into  
Each of the resulting smaller cubes is then subdivided following the procedure described in \cite[Proof of 2.10]{MR1867354}, which
divides, e.g., 
%we briefly summarize by considering the cube $[0,\frac{1}{k-1}]^n$. The key 
%Then, for the subdivision of these $n$-dimensional cubes, we take the cube $[0,\frac{1}{k-1}]^n$ again as an example and carry out the subdivision by %induction on $n$, following the procedure given in \cite[Proof of 2.10]{MR1867354}. We already discussed the subdivision for $n=2$. For $n\geq 3$, the key 
%ingredient of this procedure is the subdivision of 
$\Delta^{n-1} \times [0,\frac{1}{k-1}] $ into $n$-simplices as follows. Let $\Delta^{n-1} \times \{0\} = [v_0,\ldots \hspace{-0.02cm},v_{n-1}]$ and $\Delta^{n-1} \times \{\frac{1}{k-1}\} = [w_0,\ldots \hspace{-0.02cm},w_{n-1}]$ and note that the coordinates of $v_j \in \mathbb{R}^{n}$ and $w_j \in \mathbb{R}^{n}$ coincide in the first $n-1$ indices. 
Then, $\Delta^{n-1} \times [0,\frac{1}{k-1}]$ is given by the union of the $n$-simplices $[v_0,\ldots \hspace{-0.02cm},v_j,w_j,\ldots \hspace{-0.02cm},w_{n-1}]$, $j=0,\ldots \hspace{-0.02cm},n-1$, each intersecting the next one in an $(n-1)$-simplex face. The result of this procedure yields a
subdivision of $[0,1]^n$ into $n$-simplices with vertices $K^n$. 
As in the case $n=2$, an interpolating linear function is uniquely determined on each of these $n$-simplices \cite[Chapter~13]{smith2012linear}; we stitch the resulting linear pieces together to obtain a function $f_{\text{c}}$ satisfying Property $1$ and exhibiting the structure as demanded by Property $3$.
%\textcolor{blue}{(delete} As in the case $n=2$, we construct an interpolating linear polynomial on each of these $n$-simplices and stitch the resulting %polynomials together to obtain a function $f_{\text{c}}$ satisfying Property $1$ and exhibiting the structure as demanded by Property $3$\textcolor{blue}{)}. 
As $f_{\text{c}}$ constitutes a simplex interpolation of the discrete function $f$ on the cube $[0,1]^n$, it is continuous \cite{rovatti1998geometric}, \cite{weiser1988note}, thereby satisfying Property $2$.

%\textcolor{blue}{The function $f_{\text{c}}$ resulting from the ``subdivision and stitching’’ procedure we just described constitutes a simplex interpolation of the discrete function $f$ on the cube }

%Note that the stitching procedure makes explicit use of the fact that the $n$-simplices we construct each intersect the next one in an $(n-1)$-simplex face. This property guarantees continuity of the resulting function $f_{\text{c}}$. 

%. The same argument for the case $n=2$ can be carried over to $n\geq 3$ to show that the function~$f_{\text{c}}$ resulting from stitching all the linear polynomials %together is a well-defined continuous piecewise linear function. 
We hasten to add that, for $n\geq 2$, the set of $n$-simplices resulting from the subdivision procedure employed here is not unique, see \figref{fig:subdide2} for an illustration in the case $n=2$. But, each fixed subdivision into $n$-simplices uniquely determines a continuous piecewise linear function $f_{\text{c}}$ that interpolates $f$.

It remains to prove that the coefficients of the constituent linear polynomials %$p_1,\ldots \hspace{-0.02cm},p_\ell$,
\begin{equation}%\label{eq:coeff}
        p_j(x_1,\ldots \hspace{-0.02cm},x_n) = m_{j1}x_1+\cdots+m_{jn}x_n+\frac{b_j}{k-1}, \quad j=1,\ldots \hspace{-0.02cm},\ell,
\end{equation}
of $f_{\text{c}}$, indeed, satisfy $b_j,m_{j1},\ldots \hspace{-0.02cm},m_{jn} \in \mathbb{Z}$, for $j=1,\ldots \hspace{-0.02cm}, \ell$.
%What is left is to prove that the coefficients of $f_{\text{c}}$ are of the form in \eqref{eq:coeff}. 
As already mentioned, the coefficients of each $p_j$ are uniquely determined and obtained by solving a system of linear equations induced
by the vertex coordinates of the $n$-simplex $p_j$ resides on. We now show that the vertices $v_0,v_1,\ldots \hspace{-0.02cm},v_n\subset K^n$ of
every $n$-simplex in $[0,1]^n$ resulting from the procedure described above satisfy, after reordering if necessary,
%This will be accomplished by first showing that the vertices $v_0,v_1,\ldots \hspace{-0.02cm},v_n\subset K^n$ in
%every $n$-simplex in $[0,1]^n$ resulting from the procedure described above can be rearranged so that
\begin{equation}\label{eq:differ}
    ||v_{i}-v_{i+1}||_1 = \frac{1}{k-1}, \quad \text{for } i=0,\ldots \hspace{-0.02cm},n-1.
\end{equation}
In other words, each of the (reordered) vertices differs from the next one in only one position, specifically by $\pm \frac{1}{k-1}$. In the case
$n=1$ the condition~\eqref{eq:differ} holds without reordering. For $n=2$, no reordering is needed in $\Delta^2_1$ and $\Delta^2_2$, while in each of
$\Delta^2_3$ and $\Delta^2_4$ we only need to swap the second vertex with the first one. For general $n \geq 3$, 
%the proof follows by induction, making explicit use of the simplex set construction procedure described above. Specifically, 
we make use of the fact that,
%that \eqref{eq:differ} holds for all $n$-simplices of $[0,1]^n$. From the subdivision procedure above, 
for each $n$-simplex $\Delta^{n}$ of $[0,1]^{n}$ constructed according to the procedure described above, there exists an $(n-1)$-simplex $\Delta^{n-1} = [v_0,v_1,\ldots \hspace{-0.02cm},v_{n-1}]$ such that 
\[
\Delta^{n} =  \begin{bmatrix}
    \begin{pmatrix}
    v_0 \\ \frac{j}{k-1}
\end{pmatrix}, \ldots \hspace{-0.02cm}, \begin{pmatrix}
    v_i \\ \frac{j}{k-1}
\end{pmatrix}, \begin{pmatrix}
    v_i \\ \frac{j+1}{k-1}
\end{pmatrix},\ldots \hspace{-0.02cm},\begin{pmatrix}
    v_{n-1} \\ \frac{j+1}{k-1}
\end{pmatrix}
\end{bmatrix},
\]
for some $j\in \{0,\ldots \hspace{-0.02cm},k-2\}$ and with $i \in \{0,\ldots \hspace{-0.02cm},n-1\}$. If $\Delta^{n-1}$ satisfies~\eqref{eq:differ}, it follows directly that
$\Delta^{n}$ satisfies \eqref{eq:differ}. Now, noting that for $n=3$, $\Delta^{2}$ satisfies \eqref{eq:differ} as already established, and
proceeding by induction across $n$ establishes the desired statement.

Finally, assume that the linear piece 
\begin{equation*}
    m_1x_1+m_2x_2+\cdots +m_nx_n +m_0
\end{equation*} 
of $f_{\text{c}}$ is obtained, w.l.o.g. thanks to \eqref{eq:differ}, by interpolating $n+1$ points at 
\begin{equation*}
\begin{aligned}
    \biggl( &\frac{a_1}{k-1}, \frac{a_2}{k-1}, \ldots \hspace{-0.02cm}, \frac{a_n}{k-1}, \frac{b_0}{k-1} \biggl)  \\
   \biggl( &\frac{a_1+1}{k-1}, \frac{a_2}{k-1}, \ldots \hspace{-0.02cm}, \frac{a_n}{k-1}, \frac{b_1}{k-1} \biggl)\\
   \biggl( &\frac{a_1+1}{k-1}, \frac{a_2+1}{k-1}, \ldots \hspace{-0.02cm}, \frac{a_n}{k-1}, \frac{b_2}{k-1} \biggl)\\
  &\quad \quad \quad \quad \quad \quad \vdots \\
 \biggl( &\frac{a_1+1}{k-1}, \frac{a_2+1}{k-1}, \ldots \hspace{-0.02cm}, \frac{a_n+1}{k-1}, \frac{b_n}{k-1} \biggl),      
\end{aligned}
\end{equation*}
with $a_1,\ldots \hspace{-0.02cm}, a_n,b_0,\ldots  \hspace{-0.02cm},b_n\in \{0,1,\ldots \hspace{-0.02cm}, k-1\}$. 
The coefficients $m_0, m_1,\ldots \hspace{-0.02cm}, m_n$ hence have to satisfy the following system of linear equations
\begin{equation}\label{eq:Gauss}
    \begin{aligned}
         \frac{a_1}{k-1}m_1+ \frac{a_2}{k-1}m_2\, + &\cdots+ \frac{a_n}{k-1}m_n + m_0 = \frac{b_0}{k-1} \\
\frac{a_1+1}{k-1}m_1+ \frac{a_2}{k-1}m_2\, + &\cdots+ \frac{a_n}{k-1}m_n + m_0 = \frac{b_1}{k-1} \\
\frac{a_1+1}{k-1}m_1+ \frac{a_2+1}{k-1}m_2\, + &\cdots+ \frac{a_n}{k-1}m_n + m_0 = \frac{b_2}{k-1} \\
&\;\; \vdots \\
\frac{a_1+1}{k-1}m_1+ \frac{a_2+1}{k-1}m_2\, + &\cdots+ \frac{a_n+1}{k-1}m_n + m_0 = \frac{b_n}{k-1}.
    \end{aligned}
\end{equation}
Subtracting the first row in~\eqref{eq:Gauss} from the second and multiplying the result by $k-1$, we obtain 
\begin{equation*}
      m_1 = b_1-b_0 \in \mathbb{Z}.
  \end{equation*} 
Continuing likewise by subtracting the first row from the third yields
\begin{equation*}
      m_2  = b_2-b_0 - m_1 \in  \mathbb{Z}
%      & \;\; \vdots \\
%m_n &= b_{n}-b_0,
\end{equation*}
and so forth, establishing that $m_i \in \mathbb{Z}$, for all $i=1,2,\ldots \hspace{-0.02cm},n$.
%$m_1,\ldots \hspace{-0.02cm},m_n$ are all in $\mathbb{Z}$.
%which, owing to $b_j\in \mathbb{Z}, j=0,1,\ldots \hspace{-0.02cm},n$, establishes that the coefficients $m_1,\ldots \hspace{-0.02cm},m_n$ are all in %$\mathbb{Z}$. 
Substituting back into any of the equations in \eqref{eq:Gauss} establishes that $m_0$ is of the desired form as well and thereby finalizes the proof.
%yields 
%show that $m_0$ is of the rational value
%\begin{equation*}
%m_0 = \frac{(a_1+a_2+\cdots+a_n+1)\,b_0-(a_1b_1+a_2b_2+\cdots +a_nb_n)}{k-1}.
%\end{equation*}
\end{proof}

%\textcolor{blue}{added by Yani starts}

%\textcolor{blue}{added by Yani ends} 

We next show that functions $f_{\text{c}}$ satisfying Properties $1$-$3$ in \lemref{lma:refinefc} can, indeed, be realized by ReLU networks with integer
weights and biases in $\mathbb{Q}_{k}:=\{\frac{b}{k-1}:b\in \mathbb{Z}\}$.

\begin{lem}\label{lma:intReLU}
Let $n\in \mathbb{N}$ and let $f_{\text{c}}:[0,1]^n\rightarrow [0,1]$ be a continuous piecewise linear function with linear pieces $p_1, \ldots \hspace{-0.02cm}, p_\ell, \ell\in \mathbb{N}$, of the form
\begin{equation}\label{eq:linearpj1}
        p_j(x_1, \ldots \hspace{-0.02cm}, x_{n}) =m_{j1}x_1+\cdots+m_{jn}x_{n}+ \frac{b_j}{k-1}, \quad  j = 1,\ldots \hspace{-0.02cm},\ell,
\end{equation} where $k \geq 2$ and $m_{j1},\ldots \hspace{-0.02cm}, m_{jn}, b_j \in \mathbb{Z}$, for $j=1,\ldots \hspace{-0.02cm}, \ell$. There exists a ReLU network $\Phi \in \mathcal{N}_{n,1}$ with integer weights and biases in $\mathbb{Q}_k$, satisfying $\Phi(x) = f_{\text{c}}(x)$, for all $x\in [0,1]^n$.
\end{lem}

\begin{proof}
By~\lemref{lem:piecewise}, the function $f_{\text{c}}$ can be written in terms of $p_1,\dots,p_\ell$ through ``$\min$'' and ``$\max$'' operations in the form
\begin{equation}\label{eq:gamma}
    f_{\text{c}} = \max_I \min_{J} p_{j},
\end{equation}
where $I, J \subset \{1,\ldots \hspace{-0.02cm},\ell\}$ are index sets. The linear pieces $p_j$, for $j=1,\ldots \hspace{-0.02cm},\ell$, in~\eqref{eq:linearpj1} can straightforwardly be realized by single-layer ReLU networks with integer weights $m_{ji}$ and biases $b_{j}/(k-1)$. Next, we note that the ``$\min$'' operation can be implemented by a ReLU network with integer weights and biases equal to $0$ according to 
\begin{equation*}
    \min\{x_1,x_2\} = \rho(x_1) - \rho(-x_1) - \rho(x_1-x_2), \text{ for } x_1,x_2\in \mathbb{R}.
\end{equation*}
As 
\begin{equation*}
    \min\{x_1,x_2,x_3\} = \min\{x_1, \min\{x_2,x_3\}\}, \text{ for } x_1,x_2,x_3 \in \mathbb{R},
\end{equation*}
it follows that the minimum over any number of variables can be realized by nesting minimum operations. Hence, inspection of the proof of %\lemref{lem:ReLUConcatenating}~
\cite[Lemma II.3]{elbrachter2021deep} reveals that there exists a ReLU network with integer weights and biases equal to $0$ computing the minimum over any number of variables. Likewise, we have
\begin{equation*}
    \max\{x_1,x_2\} = \rho(x_1) - \rho(-x_1) +\rho(x_2-x_1), \text{ for } x_1,x_2\in \mathbb{R},
\end{equation*}
and 
\begin{equation*}
     \max\{x_1,x_2,x_3\} = \max\{x_1, \max\{x_2,x_3\}\}, \text{ for } x_1,x_2,x_3 \in \mathbb{R},
\end{equation*}
which analogously establishes that the operation taking the maximum over any number of variables can be realized by a ReLU network with integer weights and biases equal to $0$. Based on~\eqref{eq:gamma}, $f_{\text{c}}$ can be written as the composition of ReLU networks with integer weights and biases either equal to $0$ (for the min and max operations) or equal to $b_{j}/(k-1)$ (for the $p_{j}$). Inspection of the proof of \cite[Lemma II.3]{elbrachter2021deep} then shows that the resulting overall ReLU network will have integer weights and biases in $\mathbb{Q}_k$, as desired. 
\end{proof}

The proof of \lemref{lma:refinefc}, combined with \lemref{lma:intReLU}, reveals that even if we constrain the network $\Phi^{f}$ interpolating the transition function $f$ to have integer weights and biases in $\mathbb{Q}_k$, the network will not be unique. This is a consequence of the
$n$-simplices underlying the linear pieces $p_{j}$ constructed in the proof of \lemref{lma:refinefc} not being unique, see e.g. \figref{fig:subdide2} for an illustration of this phenomenon in the case $n=2$.
Correspondingly the DMV formulae extracted from the networks realizing the functions~$f_{\text{c}}$ based on different choices of $n$-simplices, 
following the procedure described in Section~\ref{sec:reading-out} below, will be functionally different. Beyond this nonuniqueness, we can even get interpolating ReLU networks with integer weights and biases in $\mathbb{Q}_{k}$ that do not effect simplex interpolation. 
This aspect is illustrated in the next example.

\begin{exmp}
Consider a CA with cellular space dimension $d=1$, neighborhood size $n=2$, state set $\{0,1/2,1\}$ with $k=3$, and transition function $f :\{0,1/2,1\}^2 \rightarrow \{0,1/2,1\}$ specified as follows:

\begin{center}
\vspace{0.5cm}
\begin{tabular}{|c |c |c |c | c| c| c|  } 
 \hline
$x_{-1}\,x_0$& (0,\,0) & (0,\,1/2) & (0,\,1) & (1/2,\,0) & (1/2,\,1/2) \\ \hline
$f(x_{-1},\,x_0)$ & 0 & 0 & 0 & 0 & 1    \\ \hline
\end{tabular}
\end{center}

\begin{center}
\vspace{0.5cm}
\begin{tabular}{| c| c| c|c| c| c| } 
 \hline
$x_{-1}\,x_0$&  (1/2,\,1) & (1,\,0) & (1,\,1/2) & (1,\,1) \\ \hline
$f(x_{-1},\,x_0)$  & 0 & 0& 0 & 0   \\ \hline
\end{tabular}
\end{center}
\vspace{0.5cm}
\figref{fig:eg2d} shows two different ReLU networks $\Phi^f_1,\Phi^f_2:[0,1]^2 \rightarrow [0,1]$, both with integer weights and biases in $\mathbb{Q}_{k}$, interpolating the transition function. We remark that $\Phi^f_1$ effects simplex interpolation, while $\Phi^f_2$ does not. 
%We hence see that beyond the nonuniqueness in the choice of simplices when one restricts to simplex interpolation, we can even get interpolating ReLU %networks with integer weights and biases in $\mathbb{Q}_{k}$ that do not effect simplex interpolation. 
Now, applying the DMV formula extraction algorithm introduced in \secref{sec:reading-out} below yields the DMV term associated with $\Phi^f_1$ as
\[
\tau_1 = (x_{-1}\oplus x_{-1}) \wedge (\lnot x_{-1} \oplus \lnot x_{-1}) \wedge (x_0 \oplus x_0) \wedge ( \lnot x_0 \oplus \lnot x_0),
\]
and that associated with $\Phi^f_2$ according to
\[
\tau_2=(x_{-1} \oplus x_{-1}) \wedge (\lnot x_{-1} \oplus \lnot x_{-1}) \wedge (x_0 \oplus x_0 \oplus x_0) \wedge ( \lnot x_0 \oplus \lnot x_0 \oplus \lnot x_0 ).
\]
These two DMV terms are algebraically and functionally different, but the associated term functions under $\mathcal{I}_{\text{d}}$ coincide on $K^2$, i.e.,
\[
\tau_1^{\mathcal{I}_{\text{d}}}(x_{-1},x_0)=\tau_2^{\mathcal{I}_{\text{d}}}(x_{-1},x_0) =f(x_{-1},x_0),
\]
for $(x_{-1},x_0) \in \{0,1/2,1\}^2$.

We shall consider all term functions that coincide with a given CA transition function on $K^n$ to constitute
an equivalence class in the sense of faithfully describing the logical behavior of the underlying CA.

\begin{figure}[tb]
\centering
 \includegraphics[width=10cm]{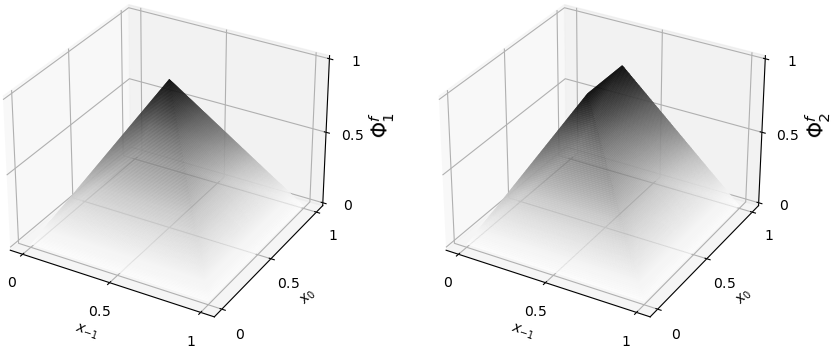}
\caption{Different ReLU networks interpolating $f$.}\label{fig:eg2d}
\end{figure}
\end{exmp}

There are two further sources of potential nonuniqueness discussed next.

\subsection{Uniqueness properties of DMV formula extraction} 

%There are four sources of nonuniqueness, which we discuss next along with possible mitigation measures.

%Reading out the CA transition function from a trained RNN is, unfortunately, not possible in a unique manner. We next develop the three underlying sources of %nonuniqueness and discuss to what extent they can be mitigated.

We first comment on a structural property pertaining to the RNN emulating the CA evolution. Recall the decomposition of $\Phi$ in~\eqref{eq:C}. Our construction built on the idea of having the hidden-state vector store the neighbors of the current input cell, leading to the specific form of $\Phi^{h}$, and making the 
%of the hidden state vector evolution network $\Phi^h$ built on the idea of having the hidden state vector store the neighbors of the current input cell and 
subnetwork $\Phi^f$ be responsible for the computation of the output samples. But this separation need not be the only way the RNN can realize the CA evolution. The resulting ambiguity is, however, easily eliminated by enforcing the split of $\Phi$ according to~\eqref{eq:C} on the RNN to be trained. In practice, this means that 
the weights of $\Phi^h$ are fixed according to \eqref{eq:HiddenNetwork} and only the weights of $\Phi^f$ are trained. In the light of \lemref{lma:refinefc} and \lemref{lma:intReLU}, the training algorithm has to enforce integer weights and biases in $\mathbb{Q}_k$ on $\Phi^f$. 
This can be done using techniques such as those described in e.g. \cite{wu2018training}, \cite{105555}.
%which we do not discuss here. 
At the end of the training process, the DMV formula underlying the CA can be read out from $\Phi^f$ using the algorithm presented in Section~\ref{sec:reading-out} below. Recall that given what was said after \lemref{lma:intReLU}, this formula will not be unique in general.

The second aspect we wish to dwell on here revolves around the fact that algebraically different DMV formulae can be functionally equivalent, just like what one has in classical algebra. 
%Two other sources of nonuniqueness stem from different algebraic expressions for the same input-output relation. Unlike the discussion above on different DMV %terms resulting from different resolutions in the input-output relation of the trained neural network, here we investigate nonuniqueness in terms of %algebraic experssions only, but the functional equivalence is actually unique. We first note that an DMV term can have many formally different expressions %with the same associated term function. For expositional clarity, we only discuss the case of binary state set $\{0,1\}$. 
To demonstrate the matter, we analyze a simple example in Boolean algebra, namely the elementary CA of rule 30 with transition function specified in DNF in~\eqref{eq:rule30logic}. A functionally equivalent, but algebraically different, expression for this transition rule is given by
%It may alternatively be expressed in DNF as
\begin{multline}\label{eq:rule30logic2}
    f_{30} = (x_{-1}\odot  \lnot x_0 \odot \lnot x_1) \oplus  (\lnot x_{-1}\odot   x_0 \odot  x_1) \oplus \\(\lnot x_{-1}\odot   x_0 \odot  \lnot x_1) \oplus (\lnot x_{-1}\odot   \lnot x_0 \odot  x_1).
\end{multline}
This shows that even if we restrict ourselves to DNF, the algebraic expression will not be unique in general.
%In other words, the corresponding Boolean formula in DNF is not unique. 
Moreover, we can also express rule 30 in conjunctive normal form (CNF), i.e., as a concatenation of a finite number of clauses linked by the Boolean $\odot$ operation, where each clause consists of a finite number of variables or negated variables connected by the Boolean~$\oplus$ operation, e.g., 
\begin{equation}\label{eq:rule30logic3}
    f_{30} = ( x_{-1}\oplus   x_0 \oplus  x_1) \odot ( \lnot x_{-1}\oplus  \lnot x_1 ) \odot (\lnot x_{-1}\oplus  \lnot x_0),
\end{equation} or 
\begin{equation}\label{eq:rule30logic4}
     f_{30} = ( x_{-1}\oplus   x_0 \oplus  x_1) \odot ( \lnot x_{-1}\oplus \lnot   x_1 ) \odot ( \lnot x_{-1}\oplus \lnot  x_0 \oplus  x_1).
\end{equation}
If we do not limit ourselves to normal forms, starting from an expression $f_{30}$ (with $x_1$ as one of its variables) for the transition function of CA~30, equivalent expressions for $f_{30}$ are given by 
\begin{equation}\label{eq:rule30logic5}
  f_{30} \oplus (x_1\odot \lnot x_1),
\end{equation}
and
\begin{equation}\label{eq:rule30logic6}
   f_{30} \odot (x_1\oplus \lnot x_1).
\end{equation}
%All Boolean logical formulae can hence be augmented following this philosophy as much as one wishes. 
Now, noting that the ReLU network realization of $f_{30}$ as expressed in~\eqref{eq:rule30logic2}, is given by
%corresponding to the different expressions for $f_{30}$ just discussed, following the methodology detailed in Section 1.2, the different Boolean expression %in~\eqref{eq:rule30logic2}-\eqref{eq:rule30logic4} lead to different ReLU networks realizing them. Concretely, we can realize $f_{30}$ by a ReLU network %$\hat{\Phi}^{f_{30}}:\mathbb{R}^3 \rightarrow \mathbb{R} $ according to~\eqref{eq:rule30logic2} as 
\begin{equation*}
   \hat{ \Phi}^{f_{30}}:= \hat{W}_3^{f_{30}} \circ \rho \circ \hat{W}_2^{f_{30}} \circ \rho \circ  \hat{W}_1^{f_{30}},
\end{equation*} with 
\begin{equation*}
 \hat{W_1}^{f_{30}}(x_{-1},x_0,x_1) = \begin{pmatrix}
    -1 & -1 & 1 \\
    1 & 1 & -1 \\
    -1 & 1 & -1 \\
    1 & -1 & -1
    \end{pmatrix} \begin{pmatrix}
        x_{1} \\x_0 \\x_{-1}
    \end{pmatrix}  + \begin{pmatrix}
        0 \\ -1 \\ 0 \\ 0
    \end{pmatrix}\text{,}
\end{equation*} 
\begin{equation*}
    \hat{ W_2}^{f_{30}}(x) = \begin{pmatrix}
    1 & 1 & 1 & 1 \\
    1 & 1 & 1 & 1
    \end{pmatrix} x + \begin{pmatrix}
   0 \\-1
    \end{pmatrix}\text{, } \quad x\in \mathbb{R}^4
\end{equation*}
\begin{equation*}
  \hat{W_3}^{f_{30}}(x) = \begin{pmatrix}1 & -1 \end{pmatrix} x, \quad x\in \mathbb{R}^2,
\end{equation*}
we can conclude, by comparing to the network~\eqref{eq:NNf30} built based on~\eqref{eq:rule30logic}, that different algebraic expressions for $f_{30}$ lead to different ReLU network realizations. Yet, these two networks must exhibit identical input-output relations. 
\iffalse
The augmentation of the Boolean formula in~\eqref{eq:rule30logic5}, by $x_1 \odot \lnot x_1 = \rho(x_1 + 1-x_1-1)=0$, amounts to composing the network $\Phi^{f_{30}}$ with a null neuron through an extra layer for the operation $\oplus$ according to~\eqref{eq:ReLUorgate}:
\[
\rho(\Phi^{f_{30}} + 0 ) - \rho(\Phi^{f_{30}} + 0-1 ); 
\]
and in~\eqref{eq:rule30logic6}, by $x_1 \oplus \lnot x_1 = \rho(x_1 + 1-x_1) -\rho(x_1 + 1-x_1-1) =\rho(1) - \rho(0)=1$, amounts to composing with a constant neuron through an extra layer for the operation $\odot$ according to~\eqref{eq:ReLUandgate}:
\[
\rho(
\Phi^{f_{30}} +1-(2-1)
).
\]
\fi 
Conversely, the observation just made shows that ReLU networks can be modified without changing their input-output relations.

\subsection{Extracting DMV formulae from trained networks}\label{sec:reading-out}

We now discuss how DMV formulae can be read out from the network $\Phi^f$. Recall that the idea is that $\Phi^{f}$ was trained on CA evolution data and the extracted DMV formula describes the logical behavior underlying this CA. In the Boolean case, with $K=\{0,1\}$, the truth table representing the CA transition function can be obtained by passing all possible neighborhood state combinations through the trained network $\Phi^f$ and recording the corresponding outputs. 
Following, e.g., the procedure in \cite[Section 12.2]{rosen2007discrete}, this truth table can then be turned into a Boolean formula.
For state sets $K$ of cardinality larger than two, we are not aware of systematic procedures that construct DMV terms from truth tables.
%A Boolean formula for the transition function can then be constructed directly from the resulting truth table following e.g. the procedure in \cite[Section %12.2]{rosen2007discrete}. 
The approach we develop here applies to state sets of arbitrary cardinality and does not work off truth tables, but rather starts from
%based on
%For state sets of cardinality larger than two, we are, however, not aware of procedures that constructs MV terms or DMV terms from given truth tables. We %next propose an approach for extracting DMV terms corresponding to CA transition functions from ReLU networks. Our procedure starts not from a truth table, %but instead 
a (trained) ReLU network~$\Phi^{f}$ that, by virtue of satisfying the interpolation property (\ref{eq:interpolation}), realizes a linear interpolation of the truth table. %Recall that ReLU networks realize piecewise linear functions.

\iffalse 
We start with a result on the fine structure of piecewise linear functions obtained by linearly interpolating CA transition rules.   
\textcolor{blue}{where Lemma 4.1 was }

\textcolor{blue}{where Proof of Lemma 4.1 was}

The structural property identified in \lemref{lma:refinefc} immediately leads to constraints on ReLU networks realizing interpolated CA transition functions. 
\fi

We start by noting that $\Phi^{f}$ encodes its underlying DMV formula both through the network topology and the network weights and biases.
\iffalse 
We proceed to discuss a procedure for extracting DMV terms from ReLU networks with integer weights and rational biases. We start by noting that a trained ReLU network stores the DMV terms not only in its weights and biases, but also in its architecture. Every piecewise linear function can be realized by a single-hidden-layer ReLU network and, conversely, a deep ReLU network can always be converted to an equivalent single-hidden-layer ReLU network without changing its input-output relation. 
Different network architectures lead to different algebraic expressions extracted (see an example in Section 4.3). 
\fi 
Honoring how the network architecture, i.e., layer-by-layer compositions of affine mappings and the ReLU nonlinearity, affects the extracted DMV term, we aim to proceed in a decompositional manner and on a node-by-node basis. It turns out, however, that the individual neurons, in general, will not correspond to DMV terms. To see this, recall that DMV term functions map $[0,1]^n$ to $[0,1]$ and note that the range of the function $\rho$ is not in $[0,1]$ in general, e.g., $\rho(3x): [0,1] \rightarrow [0,3]$. We will deal with this matter by transforming the $\rho$-neurons in $\Phi^{f}$ into $\sigma$-neurons with a suitably chosen $\sigma:\mathbb{R} \rightarrow [0,1]$. This will be done in a manner that preserves the network's input-output relation and is reversible in a sense to be made precise below.
%, i.e., the resulting $\sigma$-network can be transformed back into the $\rho$-network it originated from. 
We start by defining $\sigma$-networks.
%The next result makes this statement formal.
%and transforming $\Phi^f$ into a $\sigma$-network, which we define presently. 

\begin{defn}[$\sigma$-network]\label{defn:sigmanetwork}
    Let $L \in \mathbb{N}$ and $N_0, N_1, \ldots \hspace{-0.02cm}, N_L \in \mathbb{N}$. A $\sigma$-network $\Psi$ is a map $\Psi: \mathbb{R}^{N_0} \rightarrow \mathbb{R}^{N_L} $ given by
\begin{equation}\label{eq:defSigma}
\Psi= \begin{cases} 
      W_1, & L=1 \\
      W_2 \circ \sigma \circ W_1, & L=2 \\
      W_L \circ \sigma \circ W_{L-1} \circ \sigma \circ \cdots \circ \sigma  \circ W_1, & L \geq 3
   \end{cases} \text{,}
%    \Psi= 
%    W_L \circ \sigma \circ W_{L-1} \circ \cdots \circ \sigma  \circ W_1
%  \text{,}
\end{equation}
where, for $\ell \in \{ 1,2,\ldots \hspace{-0.02cm}, L\}$, $W_\ell: \mathbb{R}^{N_{\ell-1}}\rightarrow \mathbb{R}^{N_\ell}, W_\ell(x):= A_\ell x+b_\ell $ are affine transformations with weight matrices $A_\ell = \mathbb{R}^{N_\ell\times N_{\ell-1}}$ and bias vectors $b_\ell \in \mathbb{R}^{N_\ell}$, and the activation function $\sigma: \mathbb{R}\rightarrow [0,1], \sigma(x):= \min\{1,\max\{0,x\}\}$ acts component-wise. Moreover, we define the depth of the network $\Psi$ as $\mathcal{L}(\Psi):=L$.
%and its width as $\mathcal{W}(\Psi ):=\max_{\ell=0,\ldots \hspace{-0.02cm},L}N_\ell$.
\end{defn}

We proceed to formally establish the result on $\rho$-network to $\sigma$-network transformation announced above.
%show that transforming ReLU networks with integer weights and rational biases into equivalent $\sigma$-networks results in $\sigma$-networks that again have %integer weights and rational biases. 
 
\begin{lem}\label{lem:ConstructSigmaNetwork}
    For $n \in \mathbb{N}$, let $\Phi:[0,1]^n \rightarrow [0,1]$ be a ReLU network with integer weights and biases in $\mathbb{Q}_k$, with $k \geq 2$. There exists a $\sigma$-network $\Psi: [0,1]^n \rightarrow [0,1]$, with $\mathcal{L}(\Psi) = \mathcal{L}(\Phi)$, integer weights, and biases in $\mathbb{Q}_k$, satisfying \begin{equation*}
        \Psi(x) = \Phi(x), \quad \text{ for all } x\in [0,1]^n.
    \end{equation*} 
\end{lem}
\begin{proof} 
By~\defref{defn:ReLU}, we can write\footnote{To keep the exposition simple, we consider the form of $\Phi$ for $L \geq 3$, the cases $L=1,2$ are trivially contained in the discussion.}
\begin{equation*}
    \Phi =  W_L \circ \rho \circ W_{L-1} \circ \rho \circ \cdots \circ \rho  \circ W_1,
\end{equation*} with $N_L = 1, N_0=n, W_\ell(x) = A_\ell x+b_\ell, A_\ell\in \mathbb{Z}^{N_\ell\times N_{\ell-1}}$, and $b_\ell \in \mathbb{Q}_{k}^{N_\ell}$, 
%where $\mathbb{Q}_{k}$ denotes the set of rational numbers of the form $b/(k-1)$, 
all for $\ell\in\{1,\ldots \hspace{-0.02cm}, L\}$. We start with the conversion of neurons in the first hidden layer of $\Phi$. These neurons are all of the form 
\begin{equation*}
    \rho(a_1x+b_1), \quad \text{ with } a_1 \in \mathbb{Z}^{1\times n}, b_1\in \mathbb{Q}_{k}.
\end{equation*} As the domain of $\Phi$ is the unit cube $[0,1]^n$, $a_1x+b_1$ resides in a bounded interval in $\mathbb{R}$. Denote this interval by~$[-u, u]\subset \mathbb{R}, u\geq 0$. If $u\leq 1,$ no conversion is needed as $ \rho(x')  = \sigma(x'),$ for $x'\in [-u,u]$. For $u>1$, we note that 
%the ReLU function can be replaced by a superposition of shifted copies of $\sigma$ according to 
\begin{equation}\label{eq:relu2sigma}
     \rho(x)=  \sigma(x) + \sigma(x-1)+\cdots +\sigma(x-m), \quad x\in [-u,u],
\end{equation}
with $m\in \mathbb{N}, m\geq u-1$. In summary, we can replace each $\rho$-neuron in the first hidden layer by one or several $\sigma$-neurons according to 
\begin{equation*}
    \rho(a_1x+b_1) = \sigma(a_1x+b_1)
\end{equation*} whenever $a_1x+b_1 \leq 1$, for all $x\in [0,1]^n$, or 
\begin{equation*}
     \rho(a_1x+b_1) =\sigma(a_1x+b_1) + \cdots +\sigma(a_1x+b_1-m)
\end{equation*} when $a_1x+b_1 \leq m+1,$ for all $x\in [0,1]^n$, for some $m>1, m\in \mathbb{N}$. The resulting network $\Phi^{(1)}:[0,1]^n\rightarrow[0,1]$ is given by 
\begin{equation*}
    \Phi^{(1)} =  W_L \circ \rho \circ \cdots \circ W_3 \circ \rho \circ W_2^{(1)} \circ \sigma  \circ W_1^{(1)},  
\end{equation*} with $W_1^{(1)}(x) = A_1^{(1)}x+b_1^{(1)}, A_1^{(1)}\in \mathbb{Z}^{N_1^{(1)}\times n}, b_1^{(1)} \in \mathbb{Q}_{k}^{N_1^{(1)}},$  and $W_2^{(1)}(x) = A_2^{(1)}x+b_2, A_2^{(1)}\in \mathbb{Z}^{N_2\times N_1^{(1)}}$, where $ N_1^{(1)}$ is the number of $\sigma$-neurons in the first hidden layer of $ \Phi^{(1)}$, and satisfies 
\begin{equation*}
     \Phi^{(1)}(x) = \Phi(x), \text{ for } x\in [0,1]^n.
\end{equation*}
Noting that the input of the second-hidden layer of $\Phi^{(1)}$ is contained in the unit cube $[0,1]^{N_1^{(1)}}$, 
we can proceed as in the first step and replace each of the $\rho$-neurons in the second hidden layer by $\sigma$-neurons. 
The resulting network is given by 
\begin{equation*}
    \Phi^{(2)} =  W_L \circ  \cdots \circ W_4 \circ \rho \circ W_3^{(2)} \circ  \sigma \circ W_2^{(2)} \circ \sigma  \circ W_1^{(1)},  
\end{equation*} 
with $W_2^{(2)}(x) = A_2^{(2)}x+b_2^{(2)}, A_2^{(2)}\in \mathbb{Z}^{N_2^{(2)}\times N_1^{(1)}}, b_2^{(2)} \in \mathbb{Q}_{k}^{N_2^{(2)}}$, and $W_3^{(2)}(x) = A_3^{(2)}x+b_3, A_3^{(2)}\in \mathbb{Z}^{N_3\times N_2^{(2)}}$, where $N_2^{(2)}$ is the number of $\sigma$-neurons in the second hidden layer of $\Phi^{(2)}$, and satisfies 
\begin{equation*}
     \Phi^{(2)}(x) = \Phi(x), \text{ for } x\in [0,1]^n.
\end{equation*} Continuing in this manner, we eventually get the network $\Phi^{(L-1)}: [0,1]^n\rightarrow[0,1]$ given by 
\begin{equation*}
    \Phi^{(L-1)} =  W_L^{(L-1)} \circ \sigma \circ  W_{L-1}^{(L-1)} \circ \sigma \circ \cdots  \circ \sigma \circ W_2^{(2)} \circ \sigma  \circ W_1^{(1)},  
\end{equation*} with $W_\ell^{(\ell)}(x) = A_\ell^{(\ell)}x+b_\ell^{(\ell)}, A_\ell^{(\ell)}\in \mathbb{Z}^{N_\ell^{(\ell)}\times N_{\ell-1}^{(\ell-1)}}, b_\ell^{(\ell)} \in \mathbb{Q}_{k}^{N_\ell^{(\ell)}}$, where $N_\ell^{(\ell)}$ is the number of ($\sigma$-)neurons in the $\ell$-th hidden layer of $\Phi^{(L-1)}$, for $\ell\in\{1,\ldots \hspace{-0.02cm}, L-1\}$, and $W_L^{(L-1)}(x)= A_L^{(L-1)}x+b_L, A_L^{(L-1)} \in \mathbb{Z}^{N_L \times N_{L-1}^{(L-1)}}$. The proof is concluded by noting that the resulting $\sigma$-network $\Psi:= \Phi^{(L-1)}$ satisfies $\Psi(x) = \Phi(x)$, for $x\in [0,1]^n$, and has integer weights and biases in $\mathbb{Q}_{k}$. 
\end{proof}

The reverse transformation from a $\sigma$-network to a $\rho$-network can be effected through~\lemref{lem:sigmatorelu}. 
%We note, however, that although
%Conversely, every $\sigma$-neuron can be replaced by $\rho$-neurons making use of the relation $\sigma(x) = \rho(x) - \rho(x-1), $ for $x\in \mathbb{R}$. %This transformation keeps the weights integer and the biases rational (see~\lemref{lem:sigmatorelu}). 
Starting from a ReLU nework $\Phi$, transforming it into a $\sigma$-network $\Psi$ according to~\lemref{lem:ConstructSigmaNetwork}, and then going back to a ReLU network $\Phi'$ by using~\lemref{lem:sigmatorelu}, will recover a network $\Phi'$ that is functionally equivalent to $\Phi$, i.e., $\Phi'(x) = \Phi(x)$, for all $x\in [0,1]^n$, but can be structurally different from $\Phi$. For example, choosing different 
%the recovered ReLU network $\Phi'$ is not necessarily identical to the original network $\Phi$---in fact, different $\sigma$-networks can result %from~\lemref{lem:ConstructSigmaNetwork}, simply by choosing arbitrary 
$m \in \mathbb{N}$ with $m \geq u-1$ in~\eqref{eq:relu2sigma} would lead to such functionally equivalent, but structurally different recovered networks $\Phi'$.
%---but we only keep the functional equivalence in the transforming process, i.e., $\Phi(x) = \Psi(x) = \Phi'(x)$, for all $x\in [0,1]^n$.

We are now ready to show how DMV terms can be extracted from $\sigma$-networks. Thanks to~\lemref{lem:ConstructSigmaNetwork}, we can restrict ourselves to
$\sigma$-networks with integer weights and biases in $\mathbb{Q}_{k}$. We begin by considering a single $\sigma$-neuron, which is of the form
\begin{equation}\label{eq:sigmaneuron}
    \sigma(m_1x_1+\cdots+m_nx_n+b), \quad \text{ with } m_1,\ldots \hspace{-0.02cm},m_n\in \mathbb{Z}, b\in \mathbb{Q}_{k}. 
\end{equation} The following lemma provides an inductive way for the extraction of DMV terms from individual $\sigma$-neurons. 

\begin{lem}\label{lem:extractmv} Consider the function $f(x_1,\ldots \hspace{-0.02cm}, x_n) = m_1x_1+\cdots+m_nx_n+b,$ defined on $[0,1]^{n}$, with $ m_1,\ldots \hspace{-0.02cm},m_n\in \mathbb{Z}, b\in \mathbb{Q}_{k}.$ Without loss of generality, assume that $\max_{i=1,...,n} |m_i| = m_1$. With $f_{\circ}(x_1,\ldots \hspace{-0.02cm}, x_n) =  (m_1-1)x_1+m_2x_2+\cdots+m_nx_n+b$, we have 
    \begin{equation}\label{eq:lemma4.4}
        \sigma(f) = (\sigma(f_{\circ}) \oplus x_1) \odot \sigma(f_{\circ}+1).
    \end{equation}
\end{lem} 
\begin{proof}
   See Appendix~\ref{app:proofLem4.4} \cite{rose1958}, \cite{mundici1994constructive}.
\end{proof}

We next demonstrate, by way of a simple example, how~\lemref{lem:extractmv} can be used to extract DMV terms from individual $\sigma$-neurons. 

\begin{exmp}\label{exmp:Extraction}
We extract the DMV term underlying $ \sigma(x-2y+1/2)$ and start by noting that
\begin{align}
\quad \;  & \sigma(x-2y+1/2) \\
& = \; 1- \sigma(-x+2y-1/2+1) \label{eq:neg} \\
  & = \; \lnot\sigma(-x+2y+1/2) \label{eq:lnot} \\
   & = \; \lnot( (\sigma(-x+y+1/2)\oplus y)\odot \sigma(-x+y+3/2)), \label{eq:ApplyLem}
\end{align} 
where in~\eqref{eq:neg} we used $\sigma(x) = 1-\sigma(-x+1)$, for $x\in \mathbb{R}$, and~\eqref{eq:lnot} is by $\lnot x=1-x$. In~\eqref{eq:ApplyLem}  we applied~\lemref{lem:extractmv} with $x_1=y$. We again apply~\lemref{lem:extractmv} with $x_1=y$ to remove the $y$-terms inside $\sigma(\cdot)$ as follows: 
\begin{equation}\label{eq:Removey}
\begin{aligned}
     \sigma(-x+y+1/2) &= (\sigma(-x+1/2)\oplus y)\odot \sigma(-x+3/2) \\
       \sigma(-x+y+3/2) &= (\sigma(-x+3/2)\oplus y)\odot \sigma(-x+5/2).
\end{aligned}   
   \end{equation} Likewise, we can remove the $x$-terms inside $\sigma(\cdot)$ according to
\begin{equation}\label{eq:Removex}
   \begin{aligned}
        \sigma(-x+1/2) &= \lnot \sigma(x+1/2) =\lnot( (\sigma(1/2)\oplus x)\odot \sigma(3/2) ) = \lnot (1/2\oplus x)\\
        \sigma(-x+3/2) &= \lnot \sigma(x-1/2) =\lnot( (\sigma(-1/2)\oplus x)\odot \sigma(1/2) ) = \lnot (x \odot 1/2) \\
         \sigma(-x+5/2) &= \lnot \sigma(x-3/2) =\lnot( (\sigma(-3/2)\oplus x)\odot \sigma(-1/2) ) = 1.
   \end{aligned}     
   \end{equation} 
For the constant term $1/2$, by~\defref{defn:DivisibleMV01}, we can write
\begin{equation*}
1/2 = \delta_2 1.
\end{equation*} 
%where we denote $1:=\lnot 0$.
Substituting~\eqref{eq:Removex} back into~\eqref{eq:Removey} and then substituting the results thereof into~\eqref{eq:ApplyLem}, we obtain the DMV term corresponding to $ \sigma(x-2y+1/2)$ according to
\begin{multline}\label{eq:delta}
\lnot( (((\lnot (\delta_2 1\oplus x)\oplus y)\odot \lnot (x \odot \delta_2 1 ))\oplus y)\odot (\lnot (x \odot \delta_2 1 )\oplus y)).
\end{multline}
We conclude the example by noting that the algebraic expression of the DMV term extracted using this procedure depends on the order in which the variables are eliminated through~\lemref{lem:extractmv}; functionally these expressions will, however, all be equivalent.
\end{exmp}
%In \exmpref{exmp:Extraction} above, by iterative application of \lemref{lem:extractmv} as described, the affine transformation $(x,y)\mapsto x-2y+1/2$ %associated with the neuron $\sigma(x-2y+1/2)$ is translated into the logical operations $\lnot, \oplus,\odot,$ and $\{\delta_i\}_{i \in \mathbb{N}}$.

We are now ready to describe how the DMV term underlying a given ReLU network with integer weights and biases in $\mathbb{Q}_{k}$ mapping $[0,1]^{n}$ to $[0,1]$ can be extracted. The corresponding algorithm starts by
applying~\lemref{lem:ConstructSigmaNetwork} to convert the network into a functionally equivalent $\sigma$-network  
\begin{equation*}
 \Psi^f=  W_L \circ \sigma \circ \cdots \circ W_2 \circ \sigma \circ W_1,
\end{equation*}
where $W_1,\ldots \hspace{-0.02cm}, W_L$ are affine transformations with integer weights and biases in $\mathbb{Q}_{k}$. As the range of $\Psi^f$ is contained in $[0,1]$, we can apply the $\sigma$-activation function to the output layer without changing the mapping, i.e.,
\begin{equation*}
 \Psi^f= \sigma \circ W_L \circ \cdots \circ \sigma \circ W_2 \circ \sigma \circ W_1.
\end{equation*} Next, \lemref{lem:extractmv} is applied to each layer $\sigma \circ W_{\ell}$, $\ell = 1,\ldots \hspace{-0.02cm},L$, as in \exmpref{exmp:Extraction} above to extract the DMV terms corresponding to the individual output neurons in layer $\ell$. The resulting DMV formulae are then algebraically composed honoring the compositional structure of $\Psi^f$ to yield the overall DMV term corresponding to the ReLU network we started from. We note that in order to have the compositional structure of ReLU networks correspond to compositions in DMV algebra as just explained, it is crucial to first transform the ReLU network into a (functionally equivalent) $\sigma$-network according to~\lemref{lem:ConstructSigmaNetwork}. This is, as mentioned above, to ensure that the outputs of neurons inside the network are guaranteed to be in $[0,1]$ so as to
remain in the set ($[0,1]$) underlying the DMV algebra.
%there is a strong structural property underlying 
%acting as inputs to the second layer. 
%iteratively as \exmpref{exmp:Extraction} above to extract DMV terms from each $\sigma$-neuron across all layers, turning the affine transformations 
%$W_1,\ldots \hspace{-0.02cm}, W_L$ into $\lnot, \oplus,\odot, \{\delta_i\}_{i \in \mathbb{N}}$ operations. Finally, compose the DMV terms corresponding to %each $\sigma$-neuron according to the layered compositional structure of $\Psi^f$, we obtain an DMV associated with the network $\Phi^f$.
We have automated our extraction algorithm in Python software\footnote{Implementation available at \url{https://www.mins.ee.ethz.ch/research/downloads/NN2MV.html} }; the algorithm takes a ReLU network as input and outputs a corresponding DMV formula.
Applying this algorithm, for example, to the ReLU network $f_{\text{c}} = \rho(x_{-1}+x_0+x_1)- \rho(x_{-1}+x_0+x_1-1)$ linearly interpolating the transition function in \tblref{tbl:totalistic}, indeed, recovers the MV term $\tau = x_{-1}\oplus x_0 \oplus x_1$ we started from in \exmpref{exmp:2.1}.

\hspace{5mm}

% \section*{Acknowledgment} 
\textbf{Acknowledgment.}
The authors are deeply grateful to Prof. Olivia Caramello for drawing their attention to the McNaughton theorem and, more generally, to MV logic.

%%%%%%%%%%%%%%%%%%%%%%%%%%%%%%%%%%%%%%%%%%%%%%%%%%%%%%%%%%%%%%%%%%%%%%%%%%%%%%%%%
%%%%%%%%%%%%%%%%%%%%%%%%%%%%%%%%%%%%%%%%%%%%%%%%%%%%%%%%%%%%%%%%%%%%%%%%%%%%%%%%%
%%%%%%%%%%%%%%%%%%%%%%%%%%%%% appendix %%%%%%%%%%%%%%%%%%%%%%%%%%%%%%%%%%%%%%%%%
 \appendix
 \section{MV term corresponding to $f_{\text{c}}$ in~\eqref{eq:egf}}\label{app:A}  
 First, expand the domains of the linear pieces of $f_{\text{c}}$ to $[0,1]$ and denote them by $f_1, f_2,f_3$, with
\begin{equation*}
\begin{aligned}
    f_1(x) &= 3x,\\
    f_2(x) &= 1,\\
    f_3(x) &= -3x+3,    
\end{aligned} 
\end{equation*}
for $x\in [0,1]$. Since $f_{\text{c}}:[0,1]\rightarrow [0,1]$, we can compose each linear piece with $\sigma(\cdot)$, which results in the truncated linear functions $\sigma(f_1), \sigma(f_2),\sigma(f_3)$ depicted in \figref{fig:f123}. Next, we apply \lemref{lem:extractmv} to extract the MV terms corresponding to $\sigma(f_1), \sigma(f_2),\sigma(f_3)$.
Trivally, $\sigma(f_2) = \lnot 0$ on $[0,1]$. The function $\sigma(f_1)$ can be expressed as 
\begin{equation*}
    \sigma(f_1)(x) = x\oplus x\oplus x, \text{ for } x\in [0,1].
\end{equation*}
Likewise, we have 
\begin{equation*}
    \sigma(f_3)(x) = x\odot x\odot x, \text{ for } x\in [0,1].
\end{equation*}
Inspection of \figref{fig:f123} and $f_{\text{c}}$ in \figref{fig:examplefx} allows us to conclude that 
\begin{equation*}
    f_{\text{c}}(x) = \min\{ \sigma(f_1)(x), \sigma(f_2)(x), \sigma(f_3)(x) \}, \text{ for } x\in [0,1].
\end{equation*}
Finally, it is readily verified that the ``min'' operation can be realized according to 
\begin{equation*}
    \min\{x, y\}  = x\wedge y,
\end{equation*} and 
\[
 \min\{x, y,z\} = \min\{ \min \{x,y\},z\}.
\]
Putting the pieces together, we obtain the MV term corresponding to $f_{\text{c}}$ as
\begin{equation*}
  (x\oplus x\oplus x) \wedge  \lnot 0 \wedge (x\odot x\odot x).
\end{equation*}

\begin{figure}[h]
\centering
\resizebox{9cm}{!}{

\begin{tikzpicture}[x=1.5cm, y=1.5cm, >=stealth, scale=1, transform shape]
    \draw[->] (0,0) -- (0,2.4) node[anchor=west] {$\sigma(f_1)$};
    \draw[->] (0,0) -- (2.4,0) node[anchor= west] {$x$};
  \draw (-0.02, 2.1) -- (0.02, 2.1) node[anchor= east] {$1$}; 
      \draw (0.7,-0.02) -- (0.7,0.02) node[anchor= north] {$1/3$};
       \draw (2.1,-0.02) -- (2.1,0.02) node[anchor= north] {$1$}; 
    \node at (0.2,-0.3 cm) {$0$};
  \draw[ultra thick] (0,0) -- (0.7,2.1) node[anchor= north] {};
 \draw[ultra thick] (0.7,2.1) -- (2.1, 2.1) node[anchor= north] {};
     %%% f'
       \draw[->] (3,0) -- (3,2.4) node[anchor=west] {$\sigma(f_2)$};
     \draw[->] (3,0) -- (5.4,0) node[anchor= west] {$x$};
      \draw (2.98, 2.1) -- (3.02, 2.1) node[anchor= east] {$1$}; 
      \draw (3.7,-0.02) -- (3.7,0.02) node[anchor= north] {$1/3$};
       \draw (5.1,-0.02) -- (5.1,0.02) node[anchor= north] {$1$}; 
    \node at (3.2,-0.3 cm) {$0$};
  \draw[ultra thick] (3,2.1) -- (5.1, 2.1) node[anchor= north] {};

 \draw[->] (6,0) -- (6,2.4) node[anchor=west] {$\sigma(f_3)$};
     \draw[->] (6,0) -- (8.4,0) node[anchor= west] {$x$};
   \draw (5.98, 2.1) -- (6.02, 2.1) node[anchor= east] {$1$}; 
      \draw (7.4,-0.02) -- (7.4,0.02) node[anchor= north] {$2/3$};
       \draw (8.1,-0.02) -- (8.1,0.02) node[anchor= north] {$1$}; 
    \node at (6.2,-0.3 cm) {$0$};
 \draw[ultra thick] (6.0,2.1) -- (7.4, 2.1) node[anchor= north] {};
 \draw[ultra thick] (8.1, 0) -- (7.4, 2.1) node[anchor= north] {};
\end{tikzpicture}

}
\caption{The truncated functions $\sigma(f_1), \sigma(f_2), \sigma(f_3)$.}\label{fig:f123}
\end{figure}
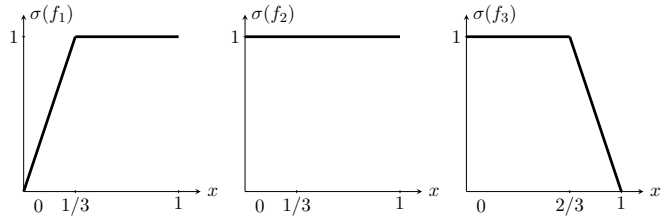

 \section{Transforming $\sigma$-networks into equivalent ReLU networks}\label{app:B}
 \begin{lem}\label{lem:sigmatorelu}
      For $d,d' \in \mathbb{N}$, let $\Psi:\mathbb{R}^d \rightarrow \mathbb{R}^{d'}$ be a $\sigma$-network with integer weights and biases in $\mathbb{Q}_{k}$. There exists a ReLU network $\Phi: \mathbb{R}^d \rightarrow \mathbb{R}^{d'}$, with $\mathcal{L}(\Phi) = \mathcal{L}(\Psi)$, integer weights, and biases in $\mathbb{Q}_{k}$, satisfying 
      \begin{equation*}
        \Phi(x) = \Psi(x), \quad \text{ for all } x\in \mathbb{R}^d.
    \end{equation*} 
    %Moreover, the network $\Phi$ has integer-valued weights and biases in $\mathbb{Q}_{k}$. 
 \end{lem}
 \begin{proof}
     By~\defref{defn:sigmanetwork}, we can write\footnote{To keep the exposition simple, we consider the form of $\Psi$ for $L \geq 3$, the cases $L=1,2$ are trivially contained in the discussion.}
\begin{equation*}
    \Psi =  W_L \circ \sigma \circ W_{L-1} \circ \sigma \circ \cdots \circ \sigma  \circ W_1,
\end{equation*} with $W_\ell(x) = A_\ell x+b_\ell, A_\ell\in \mathbb{Z}^{N_\ell\times N_{\ell-1}}, b_\ell \in \mathbb{Q}_{k}^{N_\ell},$ for $\ell\in\{1,\ldots \hspace{-0.02cm}, L\}$, $N_0=d$, and $N_L=d'$. The basic idea of the proof is to use the relationship 
%We start by converting the neurons in the first hidden layer. As
\begin{equation}\label{eq:sigm-to-rho}
    \sigma(x) = \rho(x)-\rho(x-1), \text{ for all } x \in \mathbb{R},
\end{equation}
to replace every $\sigma$-neuron in $\Psi$ with a pair of $\rho$-neurons. We start with $\sigma \circ W_{1}$ 
to obtain the equivalent network
\begin{equation*}
    \Psi^{(1)} =  W_L  \circ \sigma \circ \cdots \circ \sigma \circ W_2 \circ H_1^{(1)} \circ \rho  \circ W_1^{(1)},
\end{equation*} 
where
%we define affine transformations 
\begin{align*}
    W_1^{(1)}(x) &:= \begin{pmatrix}
        W_1(x) \\
        W_1(x) -1_{N_1}
    \end{pmatrix}, \quad x\in \mathbb{R}^{N_0},   \\
    H_1^{(1)}(x) &:= \begin{pmatrix}
    \mathbb{I}_{N_1} & -\mathbb{I}_{N_1}
    \end{pmatrix}x, \quad x\in \mathbb{R}^{2N_1},
\end{align*} 
with $1_{N_1}$ denoting the $N_1$-dimensional column vector with all entries equal to~$1$.
%to replace the first hidden layer according to
%\begin{equation*}
%    \Psi^{(1)} =  W_L  \circ \sigma \circ \cdots \circ \sigma \circ W_2 \circ H_1^{(1)} \circ \rho  \circ W_1^{(1)}.
%\end{equation*} 
It follows directly that $\Psi^{(1)}$ has integer weights and biases in $\mathbb{Q}_{k}$.
%and $\Psi^{(1)}(x)=\Psi(x)$, for all $x \in \mathbb{R}^{d}$.
Continuing in this manner, we get
%replace the $\sigma$-neurons in all layers by pairs of $\rho$-neurons. The overall resulting network
\begin{equation*}
    \Psi^{(L-1)} =  W_L  \circ H^{(L-1)}_{L-1} \circ \rho \circ W_{L-1}^{(L-1)} \circ \cdots \circ   W_{2}^{(2)} \circ H_1^{(1)} \circ \rho  \circ W_1^{(1)},
\end{equation*} with 
\begin{align*}
    W_\ell^{(\ell)}(x) &= \begin{pmatrix}
        W_\ell(x) \\
        W_\ell(x) -1_{N_\ell}
    \end{pmatrix}, \quad x\in \mathbb{R}^{N_{\ell-1}},   \\
    H_\ell^{(\ell)}(x) &= \begin{pmatrix}
    \mathbb{I}_{N_\ell} & -\mathbb{I}_{N_\ell}
    \end{pmatrix}x, \quad x\in \mathbb{R}^{2N_\ell},
\end{align*} for $\ell=1,\ldots \hspace{-0.02cm},L-1,$ satisfying
\begin{equation*}
     \Psi^{(L-1)}(x) =\Psi(x), \text{ for all } x\in \mathbb{R}^d.
\end{equation*}
The proof is concluded upon identifying $\Psi^{(L-1)}$ with $\Phi$ and noting that $\Psi^{(L-1)}$ has integer weights and biases in $\mathbb{Q}_{k}$ and
$\mathcal{L}(\Psi^{(L-1)}) = \mathcal{L}(\Psi)$.
\end{proof}

 \section{Proof of~\lemref{lem:extractmv} \cite{rose1958}, \cite{mundici1994constructive}}\label{app:proofLem4.4}
 \begin{proof}
     We follow the line of arguments in \cite{mundici1994constructive} and consider four different cases. \\
\noindent \textit{Case 1:} $f_{\circ}(x) \geq 1,$ for all $x\in [0,1]^n$. In this case, the LHS of~\eqref{eq:lemma4.4} is 
         \[
          \sigma(f)=1.
         \]
The RHS becomes 
\begin{equation*}
      (\sigma(f_{\circ}) \oplus x_1) \odot \sigma(f_{\circ}+1)  =  (1 \oplus x_1)\odot 1 = 1.
\end{equation*}

\noindent \textit{Case 2:} $f_{\circ}(x) \leq -1,$ for all $x\in [0,1]^n$. In this case, the LHS is 
     \[
          \sigma(f)=0.
         \]
The RHS is given by 
\begin{equation*}
      (\sigma(f_{\circ}) \oplus x_1) \odot \sigma(f_{\circ}+1) 
     =  (0 \oplus x_1)\odot 0=0. 
\end{equation*}

\noindent\textit{Case 3:}  $-1 < f_{\circ}(x)\leq 0$, for all $x\in [0,1]^n$. In this case, $f \in (-1,1]$ as $x_1\in [0,1]$. The RHS of~\eqref{eq:lemma4.4} becomes 
\begin{align*}
     & (\sigma(f_{\circ}) \oplus x_1) \odot \sigma(f_{\circ}+1) \\
     & =     (0 \oplus x_1)\odot (f_{\circ}+1) \\
     & = x_1\odot (f_{\circ}+1)\\
     & = \max\{0,x_1+f_{\circ}+1-1\} \\
     & = \max\{0,f\}\\
     & = \sigma(f).
\end{align*}

\noindent \textit{Case 4:}  $0 < f_{\circ}(x) < 1$, for all $x\in [0,1]^n$. In this case, $f\in (0,2)$. The RHS becomes 
\begin{align*}
     & (\sigma(f_{\circ}) \oplus x_1) \odot \sigma(f_{\circ}+1)\\
     & =    (f_{\circ} \oplus x_1)\odot 1 \\
     & = f_{\circ} \oplus x_1\\
     & = \min\{1,f_{\circ}+x_1\}\\
     & = \min \{1,f\}\\
     & = \sigma(f).
\end{align*}
 \end{proof}
\section{}

\begin{lem}\label{lem:piecewise}
    Let $n\in \mathbb{N}$ and let $f_{\text{c}}:[0,1]^n\rightarrow [0,1]$ be a continuous piecewise linear function with linear pieces 
    %$p_1, \ldots \hspace{-0.02cm}, p_\ell, \ell\in \mathbb{N}$, of the form
\begin{equation}\label{eq:linearpj}
        p_j(x_1, \ldots \hspace{-0.02cm}, x_{n}) = m_{j1}x_1+\cdots+m_{jn}x_{n}+b_j, \quad  j = 1,\ldots \hspace{-0.02cm},\ell,
\end{equation} where $m_{j1},\ldots \hspace{-0.02cm}, m_{jn} \in \mathbb{R}$, $b_j\in \mathbb{R}$. 
Let $\Pi$ be the set of all permutations of the index set $\{1,\ldots \hspace{-0.02cm},\ell\}$. There exist a nonempty set of permutations $\Sigma \subset \Pi$ and, for each $\pi \in \Sigma$, an integer $i_\pi$ with $1\leq i_\pi  \leq \ell$, such that 
\begin{equation}\label{eq:piecewiseMinMax}
    f_{\text{c}}(x) = \max_{\pi\in \Sigma}\min_{i \in \{1,\ldots \hspace{-0.02cm}, i_\pi\}} p_{\pi(i)}(x), \text{ for all } x\in [0,1]^n.
\end{equation}
\end{lem}
 \begin{proof}
 %%%Y: please add the reference to Aguzzoli's thesis as well
 The proof follows by combining arguments in the proofs of \cite[Lemma~1.4.3]{Aguzollio1998thesis} and  \cite[Proposition 9.1.4]{cignoli2013algebraic}, where the result is established for piecewise linear functions with integer coefficients. For each permutation $\pi=\pi{(1)}\pi{(2)}\cdots \pi{(\ell)} \in \Pi$, let 
     \begin{equation*}
         P_{\pi}:=\{ x\in [0,1]^n: p_{\pi(1)}(x) \geq p_{\pi(2)}(x)\geq \cdots \geq p_{\pi(\ell)}(x)\}
     \end{equation*} be the polyhedron associated with $\pi$. We refer to \cite[Lemma~1.4.2]{Aguzollio1998thesis} %%%Y: please refer to Lemma 1.4.2. in Agguzoli's thesis 
     for a proof of $P_{\pi}$, indeed, constituting a polyhedron. Define the dimension of a polyhedron as the maximum number of affinely independent points it contains minus $1$. Let $\Sigma$ be the collection of permutations whose associated polyhedra are $n$-dimensional, i.e.,
     \begin{equation*}
         \Sigma := \{\pi \in \Pi: P_\pi \text{ is $n$-dimensional}\}.
     \end{equation*}
Note that $\cup_{\pi \in \Sigma}P_\pi = [0,1]^n$ \cite[Lemma~1.4.3]{Aguzollio1998thesis}. 
%%%Y: please refer to Agguzoli's thesis Lemma 1.4.3. 
%, since for all $x\in [0,1]$ there exists a permutation $\pi$ such that $x\in P_\pi$. 
For every $\pi\in \Sigma$, let $i_{\pi}\in\{1,\ldots \hspace{-0.02cm},\ell\}$ be such that $f_{\text{c}}(x) = p_{\pi(i_{\pi})}(x),$ for $x\in P_\pi$, and let $ g_{\pi}:= \min_{i\in \{1,\ldots \hspace{-0.02cm}, i_\pi\}}p_{\pi(i)}$. For every $\pi \in \Sigma$, $g_\pi(x) = f_{\text{c}}(x) = p_{\pi(i_{\pi})}(x)$, for $x\in P_\pi$. Therefore,
\begin{equation*}
    f_{\text{c}}(x) \leq \max_{\pi \in \Sigma} g_\pi(x), \quad \text{ for } x\in [0,1]^n.
\end{equation*} 
It hence suffices to prove that $f_{\text{c}}(x)\geq g_\pi(x)$, for $x\in [0,1]^n$, for all $\pi \in \Sigma$, to obtain 
\begin{equation*}
     f_{\text{c}}(x) = \max_{\pi \in \Sigma} g_\pi(x), \quad \text{ for } x\in [0,1]^n,
\end{equation*} which establishes the desired result~\eqref{eq:piecewiseMinMax}. By continuity of $f_{\text{c}}$ and the fact that the min and max of continuous functions are continuous, it suffices to establish $f_{\text{c}} \geq g_\pi$ on the set 
$$
\{x \in [0,1]^n: \exists \pi \in \Sigma \, \mid x \,\, \text{is in the interior of $P_\pi$} \}.
$$

Now fix an arbitrary $\pi \in \Sigma$. As already noted above, $g_\pi(x) = f_{\text{c}}(x)$, for $x\in P_\pi$. Take an arbitrary point $y\notin P_\pi$ that
is in the interior of some $P_{\eta}$. There exists 
%a permutation $\eta \in \Sigma$ such that $y$ is an interior point of $P_{\eta}$, and 
a $k \in \{1,\ldots \hspace{-0.02cm}, \ell\}$ so that $f_{\text{c}}(x) = p_{\pi(k)}(x)$, for $x\in P_\eta$. We treat the cases $k \leq i_\pi$ and $k > i_\pi$ separately. First, if $k \leq i_\pi$, then 
\[
f_{\text{c}}(y) = p_{\pi(k)}(y) \geq \min_{i\in \{1,\ldots \hspace{-0.02cm}, i_\pi\}}p_{\pi(i)}(y) = g_\pi(y).
\] 
Second, assume that $k > i_\pi$. Let $x $ be a point in the interior of $P_\pi$. Then, $f_{\text{c}}(x) =  p_{\pi(i_\pi)}(x) > p_{\pi(k)}(x)$. Denote by $\mu$ the linear function connecting the points $(x,f_{\text{c}}(x))$ and $(y, f_{\text{c}}(y))$. 

Within $\{w = \lambda x + (1-\lambda)y:\mu(w) = f_{\text{c}}(w), 0<  \lambda \leq 1\}$, which is the set of intersection points of $f_{\text{c}}$ and $\mu$ over the interval $[x,y)$, let $w_0$ be the point closest to $y$ (but distinct from $y$).

%Within $\{w = \lambda x + (1-\lambda)y:\mu(w) = f_{\text{c}}(w), 0\leq  \lambda \leq 1\}$, which is the set of intersection points of $f_{\text{c}}$ and %$\mu$ over the interval $[x,y]$, among the points distinct from $y$, let $w_0$ be the point closest to $y$. 

There exists a $k' \in \{1,\ldots \hspace{-0.02cm}, \ell\}$ such that $f_{\text{c}}(w_0) = p_{\pi(k')}(w_0)$ and $f_{\text{c}}$ coincides with $p_{\pi(k')}$ on a small interval $[w_0,v] \subset [w_0,y]$, with $v \neq w_0$. By continuity of $f_{\text{c}}$, the restriction of the graph of $ p_{\pi(k')}$ to $[w_0,v]$ lies strictly below $\mu$, i.e., 
\[
p_{\pi(k')}(x) < \mu(x), \quad \text{ for } x \in [w_0,v], \,\, x \neq w_0.
\]
Moving from $w_0$ to $y$ along $p_{\pi(k')}$, we get $p_{\pi(k')}(y) < f_{\text{c}}(y)$. On the other hand, moving from $w_0$ to $x$, we obtain $p_{\pi(k')}(x) \geq f_{\text{c}}(x) = p_{\pi(i_\pi)}(x)$, thus $k' \leq i_\pi$. Hence $g_\pi(z)\leq p_{\pi(k')}(z) $, for $z\in [0,1]^n$, and in particular $g_\pi(y)\leq p_{\pi(k')}(y) < f_{\text{c}}(y)$. 

As $\pi$ and $y$ above were arbitrary, we can conclude that $f_{\text{c}}(x) \geq g_\pi(x)$, for all $ x\in [0,1]^n$, for all $\pi \in \Sigma$. The proof is complete. 
 \end{proof}
 
 \bibliographystyle{asl}
  \bibliography{reference}

\end{document}